\documentclass[conference]{IEEEtran}

\usepackage{stfloats}
\usepackage{ours}
\usepackage[numbers]{natbib}
\usepackage{multicol}
\usepackage[bookmarks=true]{hyperref}
\definecolor{gtcolor}{RGB}{0, 48, 87} 

\usepackage{amsmath,amsfonts,bm}
\usepackage{xparse}

\DeclareDocumentCommand\lf{ g g }{%
        \IfNoValueTF {#1} {\underline{f}^*} {
            \IfNoValueTF {#2} {\underline{f}^{(*)(#1)}}{\underline{f}^{(*)(#1)}_{#2}}
        }
}

\DeclareDocumentCommand\uf{ g g }{%
        \IfNoValueTF {#1} {\overline{f}^*} {
            \IfNoValueTF {#2} {\overline{f}^{(*)(#1)}}{\overline{f}^{(*)(#1)}_{#2}}
        }
}

\DeclareDocumentCommand\x{ g g }{%
        \IfNoValueTF {#1} {x} {
            \IfNoValueTF {#2} {x^{(#1)}}{x^{(#1)}_{#2}}
        }
}

\DeclareDocumentCommand\lx{ g g }{%
        \IfNoValueTF {#1} {\underline{x}} {
            \IfNoValueTF {#2} {\underline{x}^{(#1)}}{\underline{f}^{(#1)}_{#2}}
        }
}

\DeclareDocumentCommand\ux{ g g }{%
        \IfNoValueTF {#1} {\overline{x}} {
            \IfNoValueTF {#2} {\overline{x}^{(#1)}}{\overline{\vx}^{(#1)}_{#2}}
        }
}

\DeclareDocumentCommand\lvx{ g g }{%
        \IfNoValueTF {#1} {\underline{\vx}} {
            \IfNoValueTF {#2} {\underline{\vx}^{(#1)}}{\underline{\vx}^{(#1)}_{#2}}
        }
}

\DeclareDocumentCommand\uvx{ g g }{%
        \IfNoValueTF {#1} {\overline{\vx}} {
            \IfNoValueTF {#2} {\overline{\vx}^{(#1)}}{\overline{\vx}^{(#1)}_{#2}}
        }
}

\DeclareDocumentCommand\lvu{ g g }{%
        \IfNoValueTF {#1} {\underline{\vu}} {
            \IfNoValueTF {#2} {\underline{\vu}^{(#1)}}{\underline{\vu}^{(#1)}_{#2}}
        }
}

\DeclareDocumentCommand\uvu{ g g }{%
        \IfNoValueTF {#1} {\overline{\vu}} {
            \IfNoValueTF {#2} {\overline{\vu}^{(#1)}}{\overline{\vu}^{(#1)}_{#2}}
        }
}

\DeclareDocumentCommand\I{ g g }{%
        \IfNoValueTF {#1} {\mathbf{I}} {
            \IfNoValueTF {#2} {\mathbf{I}^{(#1)}}{\mathbf{I}^{(#1)}_{#2}}
        }
}

\DeclareDocumentCommand\tvx{ g g }{%
        \IfNoValueTF {#1} {\tilde{\vx}} {
            \IfNoValueTF {#2} {\tilde{\vx}^{(#1)}}{\tilde{\vx}^{(#1)}_{#2}}
        }
}

\DeclareDocumentCommand\tvu{ g g }{%
        \IfNoValueTF {#1} {\tilde{\vu}} {
            \IfNoValueTF {#2} {\tilde{\vu}^{(#1)}}{\tilde{\vu}^{(#1)}_{#2}}
        }
}

\DeclareDocumentCommand\tu{ g g }{%
        \IfNoValueTF {#1} {\tilde{u}} {
            \IfNoValueTF {#2} {\tilde{u}^{(#1)}}{\tilde{u}^{(#1)}_{#2}}
        }
}

\DeclareDocumentCommand\W{ g g }{%
        \IfNoValueTF {#1} {\mathbf{W}} {
            \IfNoValueTF {#2} {\mathbf{W}^{(#1)}}{\mathbf{W}^{(#1)}_{#2}}
        }
}

\DeclareDocumentCommand\tileW{ g g }{%
        \IfNoValueTF {#1} {\widetilde{\mathbf{W}}} {
            \IfNoValueTF {#2} {\widetilde{\mathbf{W}}^{(#1)}}{\widetilde{\mathbf{W}}^{(#1)}_{#2}}
        }
}

\DeclareDocumentCommand\tilelowerW{ g g }{%
        \IfNoValueTF {#1} {\widetilde{\mathbf{\underline{W}}}} {
            \IfNoValueTF {#2} {\widetilde{\mathbf{\underline{W}}}^{(#1)}}{\widetilde{\mathbf{\underline{W}}}^{(#1)}_{#2}}
        }
}

\DeclareDocumentCommand\tilex{ g g }{%
        \IfNoValueTF {#1} {\tilde{x}} {
            \IfNoValueTF {#2} {\tilde{x}^{(#1)}}{\tilde{x}^{(#1)}_{#2}}
        }
}

\DeclareDocumentCommand\tilelowerx{ g g }{%
        \IfNoValueTF {#1} {\underline{\tilde{x}}} {
            \IfNoValueTF {#2} {\underline{\tilde{x}}^{(#1)}}{\underline{\tilde{x}}^{(#1)}_{#2}}
        }
}

\DeclareDocumentCommand\lowerx{ g g }{%
        \IfNoValueTF {#1} {\underline{x}} {
            \IfNoValueTF {#2} {\underline{x}^{(#1)}}{\underline{x}^{(#1)}_{#2}}
        }
}

\DeclareDocumentCommand\tileupperx{ g g }{%
        \IfNoValueTF {#1} {\overline{\tilde{x}}} {
            \IfNoValueTF {#2} {\overline{\tilde{x}}^{(#1)}}{\overline{\tilde{x}}^{(#1)}_{#2}}
        }
}

\DeclareDocumentCommand\upperx{ g g }{%
        \IfNoValueTF {#1} {\overline{x}} {
            \IfNoValueTF {#2} {\overline{x}^{(#1)}}{\overline{x}^{(#1)}_{#2}}
        }
}

\DeclareDocumentCommand\tilelowerb{ g g }{%
        \IfNoValueTF {#1} {\widetilde{\mathbf{\underline{b}}}} {
            \IfNoValueTF {#2} {\widetilde{\mathbf{\underline{b}}}^{(#1)}}{\widetilde{\mathbf{\underline{b}}}^{(#1)}_{#2}}
        }
}

\DeclareDocumentCommand\tileb{ g g }{%
        \IfNoValueTF {#1} {\widetilde{\mathbf{b}}} {
            \IfNoValueTF {#2} {\widetilde{\mathbf{b}}^{(#1)}}{\widetilde{\mathbf{b}}^{(#1)}_{#2}}
        }
}

\DeclareDocumentCommand\bias{ g g }{%
        \IfNoValueTF {#1} {\mathbf{b}} {
            \IfNoValueTF {#2} {\mathbf{b}^{(#1)}}{\mathbf{b}^{(#1)}_{#2}}
        }
}

\DeclareDocumentCommand\betavar{ g g }{%
        \IfNoValueTF {#1} {\bm{\beta}} {
            \IfNoValueTF {#2} {{\bm{\beta}^{(#1)}}{}}{\bm{\beta}^{(#1)}_{#2}}
        }
}

\DeclareDocumentCommand\xivar{ g g }{%
        \IfNoValueTF {#1} {\bm{\xi}} {
            \IfNoValueTF {#2} {{\bm{\xi}^{(#1)}}{}}{\bm{\xi}^{(#1)}_{#2}}
        }
}

\DeclareDocumentCommand\xivarn{ g g }{%
        \IfNoValueTF {#1} {\bm{\xi^-}} {
            \IfNoValueTF {#2} {\bm{\xi^-}^{+(#1)}}{\bm{\xi^-}^{+(#1)}_{#2}}
        }
}

\DeclareDocumentCommand\xivarp{ g g }{%
        \IfNoValueTF {#1} {\bm{\xi^+}} {
            \IfNoValueTF {#2} {\bm{\xi^+}^{+(#1)}}{\bm{\xi^+}^{+(#1)}_{#2}}
        }
}

\DeclareDocumentCommand\nuvar{ g g }{%
        \IfNoValueTF {#1} {\bm{\nu}} {
            \IfNoValueTF {#2} {{\bm{\nu}^{(#1)}}{}}{\bm{\nu}^{(#1)}_{#2}{}}
        }
}

\DeclareDocumentCommand\hnuvar{ g g }{%
        \IfNoValueTF {#1} {\bm{\hat{\nu}}} {
            \IfNoValueTF {#2} {{\bm{\hat{\nu}}^{(#1)}}{}}{\bm{\hat{\nu}}^{(#1)}_{#2}{}}
        }
}

\DeclareDocumentCommand\muvar{ g g }{%
        \IfNoValueTF {#1} {\bm{\mu}} {
            \IfNoValueTF {#2} {{\bm{\mu}^{(#1)}}{}}{\bm{\mu}^{(#1)}_{#2}}
        }
}

\DeclareDocumentCommand\gammavar{ g g }{%
        \IfNoValueTF {#1} {\bm{\gamma}} {
            \IfNoValueTF {#2} {{\bm{\gamma}^{(#1)}}{}}{\bm{\gamma}^{(#1)}_{#2}}
        }
}

\DeclareDocumentCommand\lambdavar{ g g }{%
        \IfNoValueTF {#1} {\bm{\lambda}} {
            \IfNoValueTF {#2} {{\bm{\lambda}^{(#1)}}{}}{\bm{\lambda}^{(#1)}_{#2}}
        }
}

\DeclareDocumentCommand\tbetavar{ g g }{%
        \IfNoValueTF {#1} {{\bm{\tilde{\beta}}}} {
            \IfNoValueTF {#2} {{{\bm{\tilde{\beta}}}^{(#1)}}{}}{{{\bm{\tilde{\beta}}}^{(#1)}_{#2}}}
        }
}

\DeclareDocumentCommand\alphavar{ g g }{%
        \IfNoValueTF {#1} {\bm{\alpha}} {
            \IfNoValueTF {#2} {{\bm{\alpha}^{(#1)}}}{\bm{\alpha}^{(#1)}_{#2}}
        }
}

\DeclareDocumentCommand\d{ g g }{%
        \IfNoValueTF {#1} {\mathbf{d}} {
            \IfNoValueTF {#2} {\mathbf{d}^{(#1)}}{\mathbf{d}^{(#1)}_{#2}}
        }
}

\DeclareDocumentCommand\D{ g g }{%
        \IfNoValueTF {#1} {\mathbf{D}} {
            \IfNoValueTF {#2} {\mathbf{D}^{(#1)}}{\mathbf{D}^{(#1)}_{#2}}
        }
}

\DeclareDocumentCommand\lowerD{ g g }{%
        \IfNoValueTF {#1} {\mathbf{\underline{D}}} {
            \IfNoValueTF {#2} {\mathbf{\underline{D}}^{(#1)}}{\mathbf{\underline{D}}^{(#1)}_{#2}}
        }
}

\DeclareDocumentCommand\upperD{ g g }{%
        \IfNoValueTF {#1} {\mathbf{\overline{D}}} {
            \IfNoValueTF {#2} {\mathbf{\overline{D}}^{(#1)}}{\mathbf{\overline{D}}^{(#1)}_{#2}}
        }
}

\DeclareDocumentCommand\A{ g g }{%
        \IfNoValueTF {#1} {\mathbf{A}} {
            \IfNoValueTF {#2} {\mathbf{A}^{(#1)}}{\mathbf{A}^{(#1)}_{#2}}
        }
}

\DeclareDocumentCommand\lowerI{ g g }{%
        \IfNoValueTF {#1} {{\underline{I}}} {
            \IfNoValueTF {#2} {{\underline{I}}^{(#1)}}{{\underline{I}}^{(#1)}_{#2}}
        }
}

\DeclareDocumentCommand\upperI{ g g }{%
        \IfNoValueTF {#1} {{\overline{I}}} {
            \IfNoValueTF {#2} {{\overline{I}}^{(#1)}}{{\overline{I}}^{(#1)}_{#2}}
        }
}

\DeclareDocumentCommand\lowerW{ g g }{%
        \IfNoValueTF {#1} {\mathbf{\underline{W}}} {
            \IfNoValueTF {#2} {\mathbf{\underline{W}}^{(#1)}}{\mathbf{\underline{W}}^{(#1)}_{#2}}
        }
}

\DeclareDocumentCommand\upperW{ g g }{%
        \IfNoValueTF {#1} {\mathbf{\overline{W}}} {
            \IfNoValueTF {#2} {\mathbf{\overline{W}}^{(#1)}}{\mathbf{\overline{W}}^{(#1)}_{#2}}
        }
}

\DeclareDocumentCommand\lowerA{ g g }{%
        \IfNoValueTF {#1} {\mathbf{\underline{A}}} {
            \IfNoValueTF {#2} {\mathbf{\underline{A}}^{(#1)}}{\mathbf{\underline{A}}^{(#1)}_{#2}}
        }
}

\DeclareDocumentCommand\upperA{ g g }{%
        \IfNoValueTF {#1} {\mathbf{\overline{A}}} {
            \IfNoValueTF {#2} {\mathbf{\overline{A}}^{(#1)}}{\mathbf{\overline{A}}^{(#1)}_{#2}}
        }
}

\DeclareDocumentCommand\AA{ g g }{
        \IfNoValueTF {#1} {\mathbf{\Omega}} {
            \IfNoValueTF {#2} {\mathbf{\Omega}(#1, #1)}{\mathbf{\Omega}(#1, #2)}
        }
}

\DeclareDocumentCommand\S{ g g }{%
        \IfNoValueTF {#1} {\mathbf{S}} {
            \IfNoValueTF {#2} {\mathbf{S}^{(#1)}}{\mathbf{S}^{(#1)}_{#2}}
        }
}

\DeclareDocumentCommand\K{ g g }{%
        \IfNoValueTF {#1} {\mathbf{K}} {
            \IfNoValueTF {#2} {\mathbf{K}^{(#1)}}{\mathbf{K}^{(#1)}_{#2}}
        }
}

\DeclareDocumentCommand\B{ g g }{%
        \IfNoValueTF {#1} {\mathbf{B}} {
            \IfNoValueTF {#2} {\mathbf{B}^{(#1)}}{\mathbf{B}^{(#1)}_{#2}}
        }
}

\DeclareDocumentCommand\lowerb{ g g }{%
        \IfNoValueTF {#1} {{\mathbf{\underline{b}}}} {
            \IfNoValueTF {#2} {{\mathbf{\underline{b}}}^{(#1)}}{{\mathbf{\underline{b}}}^{(#1)}_{#2}}
        }
}

\DeclareDocumentCommand\upperb{ g g }{%
        \IfNoValueTF {#1} {\mathbf{\overline{b}}} {
            \IfNoValueTF {#2} {\mathbf{\overline{b}}^{(#1)}}{\mathbf{\overline{b}}^{(#1)}_{#2}}
        }
}

\DeclareDocumentCommand\lowerLambda{ g g }{%
        \IfNoValueTF {#1} {{\mathbf{\underline{\Lambda}}}} {
            \IfNoValueTF {#2} {{\mathbf{\underline{\Lambda}}}^{(#1)}}{{\mathbf{\underline{\Lambda}}}^{(#1)}_{#2}}
        }
}

\DeclareDocumentCommand\upperLambda{ g g }{%
        \IfNoValueTF {#1} {{\mathbf{\overline{\Lambda}}}} {
            \IfNoValueTF {#2} {{\mathbf{\overline{\Lambda}}}^{(#1)}}{{\mathbf{\overline{\Lambda}}}^{(#1)}_{#2}}
        }
}

\DeclareDocumentCommand\lowerDelta{ g g }{%
        \IfNoValueTF {#1} {{\mathbf{\underline{\Delta}}}} {
            \IfNoValueTF {#2} {{\mathbf{\underline{\Delta}}}^{(#1)}}{{\mathbf{\underline{\Delta}}}^{(#1)}_{#2}}
        }
}

\DeclareDocumentCommand\upperDelta{ g g }{%
        \IfNoValueTF {#1} {{\mathbf{\overline{\Delta}}}} {
            \IfNoValueTF {#2} {{\mathbf{\overline{\Delta}}}^{(#1)}}{{\mathbf{\overline{\Delta}}}^{(#1)}_{#2}}
        }
}

\DeclareDocumentCommand\lowerd{ g g }{%
        \IfNoValueTF {#1} {{\mathbf{\underline{d}}}} {
            \IfNoValueTF {#2} {{\mathbf{\underline{d}}}^{(#1)}}{{\mathbf{\underline{d}}}^{(#1)}_{#2}}
        }
}

\DeclareDocumentCommand\upperd{ g g }{%
        \IfNoValueTF {#1} {{\mathbf{\overline{d}}}} {
            \IfNoValueTF {#2} {{\mathbf{\overline{d}}}^{(#1)}}{{\mathbf{\overline{d}}}^{(#1)}_{#2}}
        }
}

\DeclareDocumentCommand\z{ g g }{%
        \IfNoValueTF {#1} {z} {
            \IfNoValueTF {#2} {z^{(#1)}}{z^{(#1)}_{#2}}
        }
}

\DeclareDocumentCommand\hz{ g g }{%
        \IfNoValueTF {#1} {\hat{z}} {
            \IfNoValueTF {#2} {\hat{z}^{(#1)}}{\hat{z}^{(#1)}_{#2}}
        }
}

\DeclareDocumentCommand\hx{ g g }{%
        \IfNoValueTF {#1} {\hat{x}} {
            \IfNoValueTF {#2} {\hat{x}^{(#1)}}{\hat{x}^{(#1)}_{#2}}
        }
}

\DeclareDocumentCommand\hu{ g g }{%
        \IfNoValueTF {#1} {\hat{u}} {
            \IfNoValueTF {#2} {\hat{u}^{(#1)}}{\hat{u}^{(#1)}_{#2}}
        }
}

\DeclareDocumentCommand\bu{ g g }{%
        \IfNoValueTF {#1} {\mathbf{u}} {
            \IfNoValueTF {#2} {\mathbf{u}^{(#1)}}{\mathbf{u}^{(#1)}_{#2}}
        }
}

\DeclareDocumentCommand\bl{ g g }{%
        \IfNoValueTF {#1} {\mathbf{l}} {
            \IfNoValueTF {#2} {\mathbf{l}^{(#1)}}{\mathbf{l}^{(#1)}_{#2}}
        }
}

\DeclareDocumentCommand\aaa{ g }{%
        \IfNoValueTF {#1} {\bm{a}} {
            {\bm{a}^{({#1})}}
        }
}

\DeclareDocumentCommand\haaa{ g }{%
        \IfNoValueTF {#1} {\bm{\hat{a}}} {
            {\bm{\hat{a}}^{({#1})}}
        }
}

\DeclareDocumentCommand\bbb{ g g }{%
        \IfNoValueTF {#1} {\mathbf{P}} {
            \IfNoValueTF {#2} {{\mathbf{P}_{#1}}}{{\mathbf{P}_{#1}^{({#2})}}}
        }
}

\DeclareDocumentCommand\hbbb{ g g }{%
        \IfNoValueTF {#1} {\mathbf{\hat{P}}} {
            \IfNoValueTF {#2} {{\mathbf{\hat{P}}_{#1}}}{{\mathbf{\hat{P}}_{#1}^{({#2})}}}
        }
}

\DeclareDocumentCommand\ccc{ g g }{%
        \IfNoValueTF {#1} {\mathbf{q}} {
            \IfNoValueTF {#2} {{\mathbf{q}_{#1}}}{{\mathbf{q}_{#1}^{(#2)}}{}}
        }
}

\DeclareDocumentCommand\constc{ g }{%
        \IfNoValueTF {#1} {c} {
            {c^{({#1})}}
        }
}

\DeclareDocumentCommand\setz{ g g }{%
        \IfNoValueTF {#1} {\mathcal{Z}} {
            \IfNoValueTF {#2} {\mathcal{Z}^{(#1)}}{\mathcal{Z}^{(#1)}_{#2}}
        }
}

\DeclareDocumentCommand\setzp{ g g }{%
        \IfNoValueTF {#1} {\mathcal{Z^+}} {
            \IfNoValueTF {#2} {\mathcal{Z}^{+(#1)}}{\mathcal{Z}^{+(#1)}_{#2}}
        }
}

\DeclareDocumentCommand\setzn{ g g }{%
        \IfNoValueTF {#1} {\mathcal{Z^-}} {
            \IfNoValueTF {#2} {\mathcal{Z}^{-(#1)}}{\mathcal{Z}^{-(#1)}_{#2}}
        }
}

\DeclareDocumentCommand\tsetz{ g g }{%
        \IfNoValueTF {#1} {\tilde{\mathcal{Z}}} {
            \IfNoValueTF {#2} {\tilde{\mathcal{Z}}^{(#1)}}{\tilde{\mathcal{Z}}^{(#1)}_{#2}}
        }
}

\DeclareDocumentCommand\tz{ g g }{%
        \IfNoValueTF {#1} {\tilde{z}} {
            \IfNoValueTF {#2} {\tilde{z}^{(#1)}}{\tilde{z}^{(#1)}_{#2}}
        }
}

\DeclareDocumentCommand\f{ g g }{%
        \IfNoValueTF {#1} {f} {
            \IfNoValueTF {#2} {f^{(#1)}}{f^{(#1)}_{#2}}
        }
}

\def\Figref#1{Figure~\ref{#1}}

\def\Secref#1{Section~\ref{#1}}
\def\eqref#1{Eq.~\ref{#1}}
\def\Algref#1{Algorithm~\ref{#1}}

\def\Theoref#1{Theorem~\ref{#1}}

\def\Twoprobref#1#2{Problems \ref{#1} and \ref{#2}}

\def\ceil#1{\lceil #1 \rceil}

\def\1{\bm{1}}

\def\vu{{\bm{u}}}

\def\vx{{\bm{x}}}
\def\vy{{\bm{y}}}

\DeclareMathAlphabet{\mathsfit}{\encodingdefault}{\sfdefault}{m}{sl}
\SetMathAlphabet{\mathsfit}{bold}{\encodingdefault}{\sfdefault}{bx}{n}

\def\gB{{\mathcal{B}}}

\def\gD{{\mathcal{D}}}

\def\gL{{\mathcal{L}}}

\def\gR{{\mathcal{R}}}
\def\gS{{\mathcal{S}}}
\def\gT{{\mathcal{T}}}
\def\gU{{\mathcal{U}}}

\def\gX{{\mathcal{X}}}

\def\gZ{{\mathcal{Z}}}

\def\sR{{\mathbb{R}}}

\newcommand{\dtdyn}{f_{\text{dt,dyn}}}
\newcommand{\ctdyn}{f_{\text{ct,dyn}}}
\newcommand{\ctctl}{f_{\text{ct,ctl}}}
\newcommand{\refx}{x_{\text{ref}}}
\newcommand{\refu}{u_{\text{ref}}}
\newcommand{\refy}{y_{\text{ref}}}

\newcommand{\refvy}{\vy_{\text{ref}}}
\newcommand{\ctlu}{u_{\text{ctl}}}

\begin{document}

% paper title
\title{\looseness-1\fontsize{23}{24} \selectfont Parallel Differentiable Reachability for Learning and Planning with Certified Neural Dynamics and Controllers}

\author{\authorblockN{Keyi Shen and Glen Chou}
\authorblockA{Georgia Institute of Technology, Atlanta, GA 30308\\
Email: \texttt{\{kshen84, chou\}@gatech.edu}\\
\textcolor{gtcolor}{\textbf{Project Site and Video}: \href{https://trustworthyrobotics.github.io/diffreach_site/}{(Link)}} \quad $\mid$ \quad \textcolor{gtcolor}{\textbf{Code (DiffReach)}: \href{https://github.com/trustworthyrobotics/DiffReach}{(Github)}} \quad $\mid$ \quad \textcolor{gtcolor}{\textbf{Code (DiffReach-Robotics)}: \href{https://github.com/trustworthyrobotics/DiffReach-Robotics}{(Github)}}
\vspace{-8pt}}}

\maketitle

\begin{abstract}
\looseness-1Neural network (NN) dynamics models and control policies achieve strong performance in robotics, but providing sound guarantees under uncertainty remains difficult, especially for closed-loop NN systems. Existing reachability tools provide formal over-approximations, yet are often non-differentiable, overly conservative, or too slow for modern learning and online planning pipelines. To address this, we present a parallelizable, differentiable reachability framework in JAX for continuous- and discrete-time systems with analytical and NN-based dynamics and controllers. Our framework combines Taylor-model flowpipe construction with CROWN-style linear bound propagation through a unified representation that preserves affine dependencies while supporting GPU-batched computation and automatic differentiation.
Building on this reachability primitive, we develop (i) a certified training method that encourages reachability-friendly dynamics models and controllers, and (ii) a reachability-aware sampling-based MPC scheme with gradient-based refinement. Experiments on non-prehensile manipulation and quadrotor tasks, including hardware and higher-dimensional evaluations (up to 72D), demonstrate practical online planning while maintaining certified reachable-set over-approximations under bounded uncertainty. %Our implementation is avaiable at \url{https://github.com/trustworthyrobotics/DiffReach-Robotics}.
\end{abstract}

\IEEEpeerreviewmaketitle

\begin{figure*}[!b]
    \centering\vspace{-15pt}
    \includegraphics[width=\linewidth]{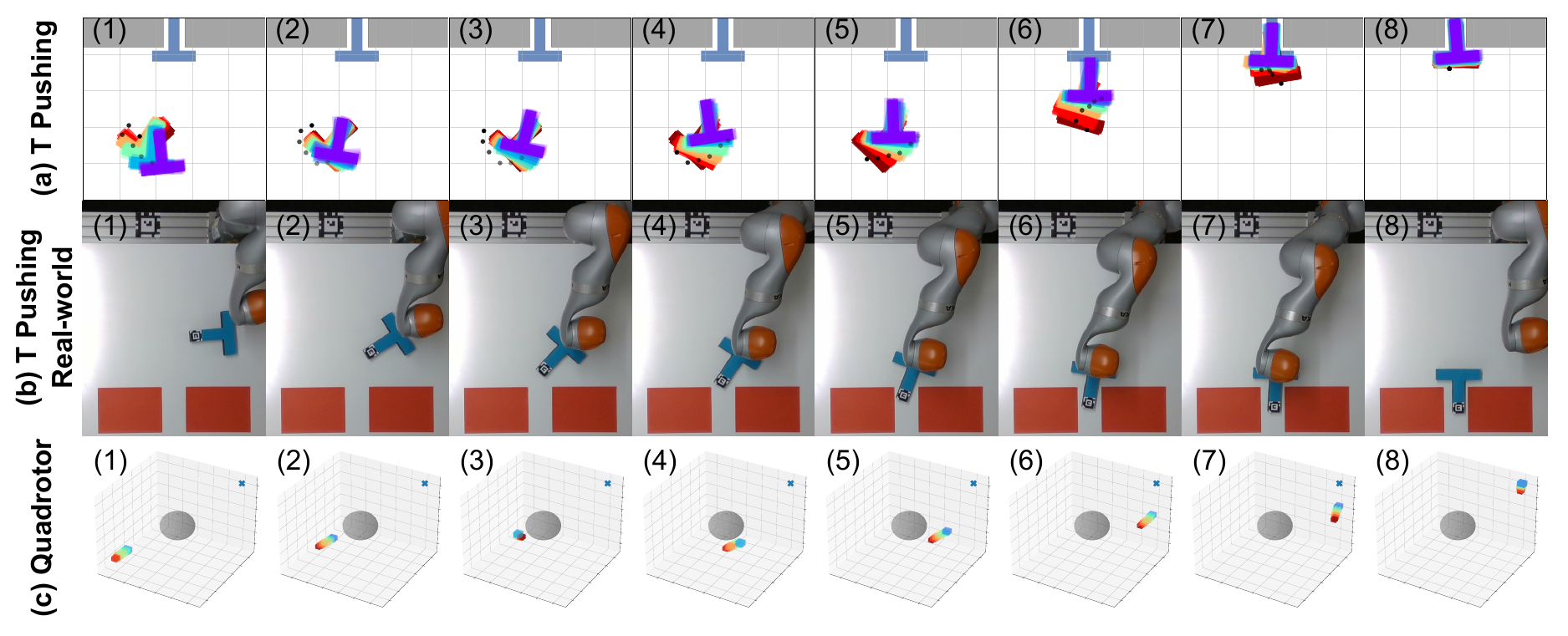}\vspace{-14pt}
    \caption{\looseness-1\textbf{Reachability-guided planning}. 
    We visualize three successful model predictive control (MPC) rollouts of our reachability-guided planner using certified, reachability-regularized models on \textbf{(a)} T-pushing, \textbf{(b)} T-pushing (Real-world) and \textbf{(c)} quadrotor. At each replanning step (indices shown), we plot the current state, the planned trajectory, and the certified reachable set under bounded disturbance (shaded footprints for T-pushing and 3D boxes for quadrotor). The results show that our planner achieves the task while maintaining tight, step-wise formal reachability guarantees under uncertainty. Overall, our parallel differentiable reachability tool can be efficiently integrated into model-training and planning pipelines in robotics, enabling verified-by-construction learning and planning with learned models.\vspace{-6pt}
    }
    
    % \caption{\looseness-1\textbf{Benchmarking certified training}. \textbf{Top}: Across systems, CT/DT, and tasks, our method trains certified dynamics and controllers with negligible performance loss compared to na\"ive training. \textbf{Bottom}: After certified training, the reachable sets computed with our approach are substantially less conservative compared to those computed with the na\"ive models, closely matching a lower bound on the volume (``sampled").}
    \label{fig:planning}
\end{figure*}

\vspace{-3pt}
\section{Introduction}
\vspace{-3pt}

\looseness-1Neural network (NN) dynamics models and controllers are increasingly used in robotics~\citep{li2018learning,manuelli2020keypoints,wu2023daydreamer,shi2023robocook}, enabling operation in complex environments where analytical modeling is infeasible. Despite strong empirical performance, providing robustness and safety guarantees for NN-in-the-loop systems remains challenging, as NNs introduce high-dimensional nonlinear components that complicate worst-case analysis under disturbances and uncertainty~\citep{manzanas_lopez2024arch}. Classical reachability analysis provides formal guarantees through reachable-set over-approximations~\citep{Althoff2015ARCH, chen2013flow}, but often scales poorly and struggles with NN components, leading to conservative bounds. Modern neural network verification (NNV) methods, such as CROWN~\citep{zhang2018efficient,xu2020automatic} and abstract interpretation~\citep{singh2019abstract}, scale to large networks and provide efficient bound propagation, but are primarily designed for static verification and do not naturally handle continuous-time dynamics. Recent efforts combining reachability analysis and NNV~\citep{manzanas_lopez2024arch} support NN-controlled systems, but remain largely offline, non-differentiable, and difficult to integrate into modern robot learning and online planning. As a result, reachability analysis is typically restricted to a \textit{post-hoc} certification mechanism, and these limitations prevent verification from \textit{directly informing} the training and deployment of certifiable models in robot learning. This can lead to costly cycles of non-certified training, failed certification, and retraining.

\looseness-1To address these challenges, we develop an efficient \textbf{parallel and differentiable reachability framework} that integrates certified analysis into robot learning and planning. Examples of reachability-aware planning rollouts are shown in~\Figref{fig:planning}. Our key idea is to unify Taylor-model-based reachability for continuous-time dynamics with CROWN-style NN verification within a single JAX~\citep{jax2018github} computation graph. Our reachability tool, \textit{DiffReach}, combines Taylor-model flowpipe construction for analytical dynamics with CROWN-style linear bound propagation for neural components through a \emph{unified representation} that preserves affine dependencies \emph{without intermediate interval relaxation}. To enable efficient GPU execution and automatic differentiation, reachability propagation is reformulated using \emph{fixed-shape tensor operations} compatible with just-in-time (JIT) compilation and batched computation. The resulting framework supports continuous- and discrete-time systems with analytical and NN-based dynamics and controllers while exposing differentiable reachable-set objectives for downstream optimization. Building on this reachability primitive, we develop \textit{DiffReach-Robotics}, which consists of 1) a \textbf{certified training} method for learning neural dynamics and controllers and 2) a \textbf{reachability-aware sampling-based model predictive control (MPC)} scheme with gradient-based refinement. Experiments on manipulation and quadrotor tasks, including hardware and higher-dimensional evaluations, demonstrate practical online planning while maintaining certified reachable-set over-approximations under bounded uncertainty. Our contributions are:

\begin{itemize}
    \item A \textit{parallel and differentiable reachability framework} combining Taylor-model flowpipes and CROWN-style bound propagation within a unified JAX-based representation for continuous- and discrete-time systems with analytical and NN-based dynamics and controllers.
    \item A \textit{certified training method} that leverages differentiable reachable-set objectives to improve robustness of neural components under uncertainty.
    \item A \textit{reachability-aware sampling-based MPC scheme} that integrates certified reachable sets into online planning with gradient-based refinement.
    \item Experimental evaluation on manipulation and quadrotor tasks, including hardware and higher-dimensional systems (up to 72D), demonstrating practical integration of certified reachability into robot learning and planning.
\end{itemize}

% \vspace{-4pt}
\section{Related Work}
% \vspace{-3pt}

\noindent\textbf{Robust planning and control in robotics.\quad}
Robotic tasks such as manipulation and agile flight involve complex nonlinear dynamics that make robust control under uncertainty challenging. A common approach separates \textit{high-level planning} from \textit{low-level feedback control}~\citep{singh2018robust,herbert2017fastrack,yin2020optimization}. Existing methods compute invariant tubes or safety certificates around nominal trajectories using sums-of-squares~\citep{singh2018robust,majumdar2017funnel}, Hamilton-Jacobi reachability~\citep{herbert2017fastrack}, or contraction-based analysis~\citep{chou2021model,knuth2022statistical,chou2022safe,singh2023robust,knuth2021planning}. While these approaches provide strong guarantees for nonlinear systems, they are typically designed for fixed analytical controllers and offline analysis. Extensions to NN controllers~\citep{knuth2021planning,chou2021model,srinivasan2026safety} similarly focus on post-hoc verification rather than integration into controller learning.

\noindent\textbf{NN dynamics and controllers.\quad} 
NN dynamics models learned from simulation or real-world data are widely used in robotic manipulation and control~\citep{shi2023robocook, Wang-RSS-23}, including models based on visual observations~\citep{finn2016unsupervised,ebert2018visual}, latent representations~\citep{hafner2019dream,wu2023daydreamer}, and structured state representations~\citep{manuelli2020keypoints,shi2022robocraft}. Despite strong empirical performance, most planning methods using NN dynamics lack formal safety guarantees, while certified approaches typically rely on slow offline verification~\citep{knuth2021planning, chou2021model} decoupled from model training and planning. Learned NN controllers~\citep{levine2016end, chi2025diffusion} and heuristic safety filters~\citep{liu2023safe, zhang2023exact, so2024train} similarly lack optimization-aware certified objectives. In contrast, our framework directly incorporates differentiable reachable-set objectives into learning and planning.

\noindent\textbf{Reachability and NNV.\quad}
\looseness-1NNV provides formal guarantees on network outputs. Early methods~\citep{tjeng2019evaluating,bunel2018unified,lu2020neural} struggled to scale due to limited parallelization and GPU support~\citep{salman2019convex,pmlr-v162-zhang22ae,liu2021algorithms}. Bound-propagation methods~\citep{wong2018provable,singh2019abstract,wang2018efficient,gowal2018effectiveness}, including CROWN~\citep{zhang2018efficient}, improve scalability by propagating certified affine bounds through network layers. Practical tools such as \abcrown~\citep{xu2020fast,wang2021beta,zhang2022gcpcrown} achieve strong performance on large networks~\citep{brix2024fifth,kaulen20256th}. NNV has also been applied to robotic control, including Lyapunov and Zubov neural control~\citep{yang2024lyapunov,11312100,li2025two}, neural contraction metrics~\citep{li2025neural}, and branch-and-bound planning with neural dynamics~\citep{shen2024bab}, but these methods primarily certify stability certificates or planning objectives rather than differentiable reachable-set propagation for closed-loop NN systems.

Classical reachability tools address dynamical systems directly. CORA~\citep{Althoff2015ARCH}, JuliaReach~\citep{bogomolov2019juliareach}, and NNV~\citep{tran2020nnv,lopez2023nnv} support reachability analysis for nonlinear and NN-controlled systems\citep{ARCH-COMP24}, while Flow*~\citep{chen2012taylor,chen2013flow,chen2016decomposed,chen2015reachability} develops rigorous Taylor-model flowpipe construction for continuous-time systems~\citep{berz1998verified}. Flow*-based NNCS tools such as POLAR~\citep{huang2022polar,wang2023polar} and CROWN-Reach~\citep{CROWNReach2026} further combine Taylor-model reachability with NN bound propagation for certified NN-in-the-loop analysis. Recent JAX-based tools, including immrax~\citep{harapanahalli2024immrax} and GoTube~\citep{gruenbacher2022gotube}, improve compatibility with GPU acceleration and automatic differentiation, though immrax relies on conservative interval propagation and GoTube provides statistical rather than formal guarantees. While these methods provide strong foundations for certified NN system analysis, they are primarily designed for offline verification rather than differentiable reachability integrated into learning and online planning.

\looseness-1Certified training and verification-aware learning have also been explored in NN verification~\citep{zhang2019towards,huang2021training} and neural control policies~\citep{shi2026certifiedtrainingbranchandboundlyapunovstable, wu2024verifiedsafereinforcementlearning}. In contrast, our framework unifies Taylor-model reachability and scalable NN bound propagation within a differentiable JAX-based framework, enabling reachable-set objectives to directly inform robot learning and online planning.

\vspace{-2pt}
\section{Problem Statement}
\vspace{-2pt}

\looseness-1We consider robotic systems that combine high-level planning with low-level feedback control under bounded uncertainty, where both dynamics models and controllers may contain NN components. Our goal is to compute sound reachable-set over-approximations over finite horizons for both discrete-time planning dynamics and continuous-time closed-loop execution, while exposing differentiable reachable-set objectives that can be integrated into downstream learning and online planning. 
\vspace{-2pt}
\subsection{Planning under bounded uncertainty}
\looseness-1We follow a modular planning-control framework, where a high-level planner optimizes a sequence of actions using a predictive discrete-time dynamics model. Let the planning state and action be \(x\in\sR^n\) and \(u\in\sR^k\), respectively, and consider the predictive dynamics
\(
\hx_{t+1}=\dtdyn(\hx_t,u_t),
\)
where \(\dtdyn\) may be an analytical model or a neural network (NN). Starting from the current state \(\hx_{t_0}=x_{t_0}\), the planner rolls out \(\dtdyn\) over a horizon \(H\) to obtain predicted future states \(\{\hx_t\}_{t=t_0}^{t_0+H}\).

\looseness-1In practice, the true state may deviate from the nominal estimate due to sensing errors, modeling errors, or external disturbances. We model the initial uncertainty as the bounded set
\(
x_{t_0}\in \gB_{\epsilon}(\hx_{t_0})
=\{x\mid \hx_{t_0}-\epsilon \le x \le \hx_{t_0}+\epsilon\}
\) where \(\epsilon\in\sR\) is the perturbation radius. 
Given a fixed action sequence \(\vu:=\{u_t\}_{t=t_0}^{t_0+H}\), we denote by
\(
\gR_t(\vu,\gB_{\epsilon}(\hx_{t_0}))
\)
the reachable set at time \(t\), i.e., the set of all states reachable from initial states in \(\gB_{\epsilon}(\hx_{t_0})\) under the dynamics \(\dtdyn\) and action sequence \(\vu\). In our method, \(\gR_t\) is represented by a sound over-approximation that accounts for bounded uncertainty propagation through the dynamics. 
We consider the following uncertainty-aware finite-horizon trajectory optimization problem:
\begin{equation}\label{eq:trajop_unc}
\begin{aligned}
    \min_{\vu}\quad&\textstyle\sum_{t=t_0}^{t_0+H} c(\hx_t, u_t) \\
    \text{s.t.}\quad 
    &\hx_{t+1} = \dtdyn(\hx_t, u_t), \\
    &G\!\left(\gR_t(\vu,\gB_{\epsilon}(\hx_{t_0}))\right)\ge 0,\quad t=t_0,\dots,t_0+H,
\end{aligned}
\end{equation}
\looseness-1where \(c\) is a task-dependent stage cost, \(H\) is the planning horizon, and actions are box-constrained to \(u_t\in\gU:=\{u\mid \underline{u}\le u\le \overline{u}\}\subset\sR^k\). The generalized constraint functional \(G(\cdot)\) enforces robustness and safety over reachable sets, e.g., collision avoidance for all states in \(\gR_t\) or penalties on reachable-set size.

The central challenge is to compute tight and sound reachable-set over-approximations efficiently enough for repeated use inside online MPC and trajectory optimization.

\subsection{Low-level control under bounded uncertainty}
We now consider low-level feedback control for continuous-time execution. Let
\(
\ctlu=\ctctl(\hx,\refy)\in \sR^{l}
\)
denote a neural controller that takes as input the current state estimate \(\hx\) and a reference signal \(\refy\) (e.g., desired states or actions from the high-level planner), and outputs a fine-grained control input applied to the continuous-time dynamics
\begin{equation}
\dot{x}(\tau)=\ctdyn(x(\tau),\ctlu).
\end{equation}
We assume the dynamics \(\ctdyn\) are available for reachability analysis, either as an analytical ODE (e.g., quadrotor dynamics) or as a learned neural ODE when analytical models are unavailable for complex manipulation systems.

The controller operates at a fixed control frequency \(q_{\text{ctl}}=1/\delta\), typically higher than the planning frequency. Over each control interval \(\tau\in[i\delta,(i+1)\delta]\), we apply a zero-order hold policy: the control input \(\ctlu\) and reference \(\refy\) are treated as constant during the interval, yielding a closed-loop trajectory determined by the ODE solution with initial condition \(x(i\delta)\).

As in the planning setting, the true state may deviate from its nominal estimate due to bounded uncertainty, i.e.,
\(
x(0)\in\gB_{\epsilon}(\hx(0)).
\)
Given a reference sequence \(\refvy\), we seek a sound over-approximation of the closed-loop reachable set
\begin{equation}
\gR(\tau\,;\refvy,\gB_{\epsilon}(\hx(0))),
\end{equation}
i.e., the set of all states reachable under the initial uncertainty and closed-loop execution of \(\ctctl\) conditioned on \(\refvy\).

Beyond post-hoc verification, we use reachable-set information directly during learning. In particular, differentiable summaries of \(\gR(\tau\,;\cdot)\) are incorporated as regularizers or constraints when training controllers and, when applicable, continuous-time dynamics models. This requires the reachability computation to support repeated invocation during optimization while exposing differentiable objectives that can guide parameter updates. 

In summary, we target three problems:
\vspace{-3pt}
\begin{problem}\label{prob:plan}
\textbf{(Reachability-aware planning)}
Given a predictive dynamics model \(\dtdyn\), stage cost \(c\), robustness/safety functional \(G\), nominal initial state \(\hx_0\), disturbance radius \(\epsilon\), and planning horizon \(H\), compute an action sequence \(\vu\) that achieves the task objective while maintaining sound reachable-set over-approximations \(\{\gR_t\}_{t=0}^{H}\) under the initial uncertainty \(x_0\in\gB_{\epsilon}(\hx_0)\). The reachable-set computation should be efficient to support repeated online replanning within MPC.
\end{problem}
\vspace{-4pt}
\begin{problem}\label{prob:learn}
\textbf{(Certified learning of dynamics and control)}
Given a dataset \(\gD\) of \(N\) state--action trajectories and a disturbance radius \(\epsilon\), learn model components (e.g., \(\dtdyn\), \(\ctctl\), and \(\ctdyn\) when applicable) that achieve strong nominal performance while remaining robust to bounded uncertainty, as quantified by reachable-set-based objectives and constraints.
\end{problem}
\vspace{-4pt}
\begin{problem}\label{prob:reach}
\textbf{(Differentiable reachability engine)}
Design a reachability engine that supports the planning and learning problems in~\Twoprobref{prob:plan}{prob:learn}. The engine should handle both continuous- and discrete-time systems with NN-based dynamics and controllers, while supporting exogenous reference inputs such as planned trajectories and tracking references. It should provide sound reachable-set over-approximations efficiently enough for repeated use in online planning and iterative learning, while exposing differentiable reachable-set objectives for downstream optimization.
\end{problem}

% \begin{itemize}
%     \item Inputs
%     \item Outputs
%     \item Problem 1: design a fast, differentiable (not too conservative) reachability tool that is amenable to integration in machine learning pipelines for training certified dynamics and controllers
%     \item Problem 2: design these controllers using the differentiable tool, and use them for robust planning and control.
% \end{itemize}

\begin{figure*}[t]
    \centering%\vspace{-15pt}
    \includegraphics[width=\linewidth]{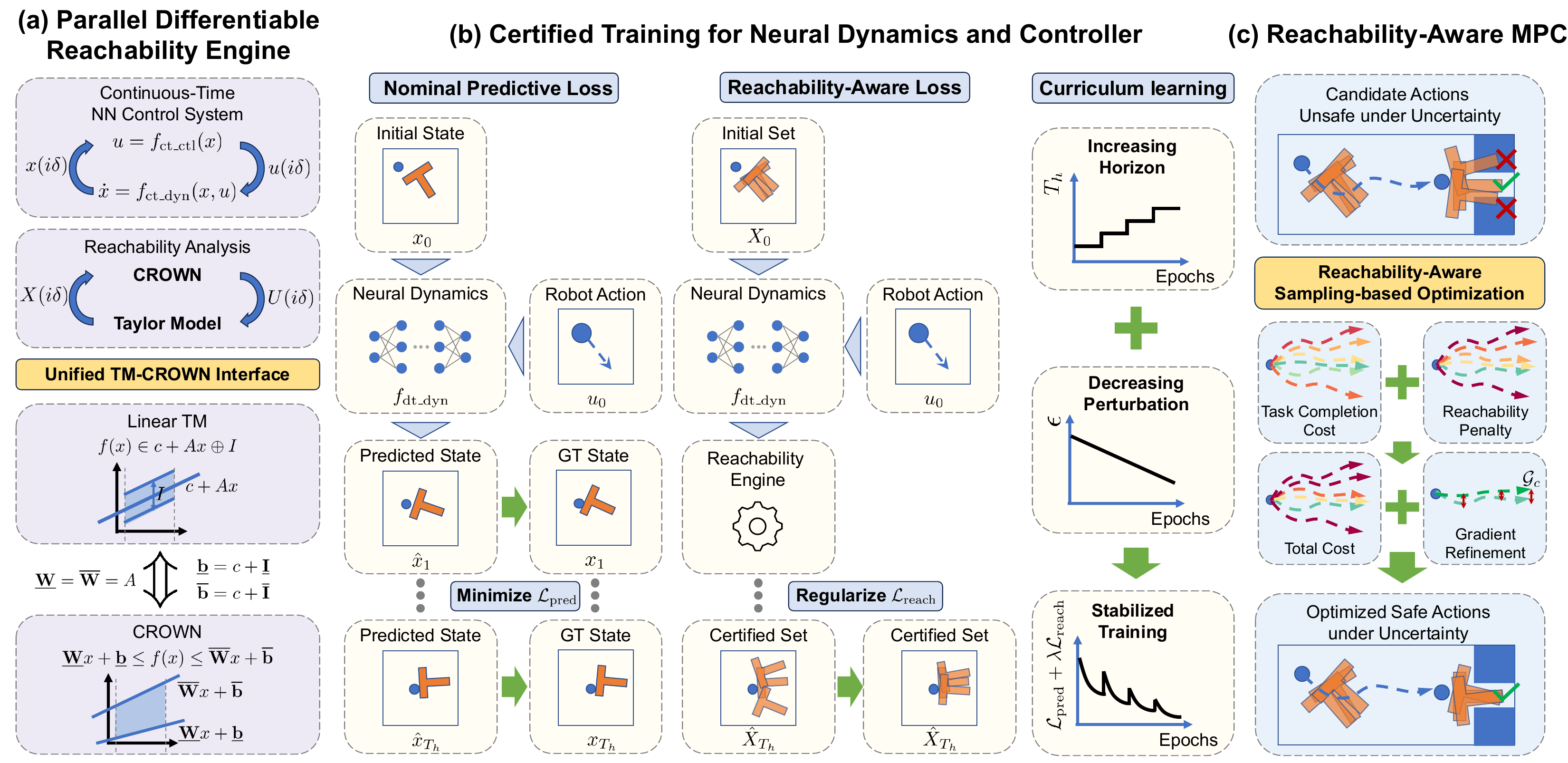}\vspace{-6pt}
    \caption{\textbf{Overview of the proposed framework.} 
\textbf{(a)} Our parallel differentiable reachability engine unifies Taylor-model (TM) flowpipe propagation for analytical dynamics with CROWN-based neural bound propagation through a shared TM--CROWN interface, enabling GPU-batched and differentiable reachability analysis for continuous- and discrete-time systems. 
\textbf{(b)} The differentiable reachable sets are used to enable certified training of neural dynamics and controllers. Standard prediction or tracking losses are augmented with reachability-aware regularization that penalizes large certified reachable sets under bounded uncertainty, while curriculum learning stabilizes optimization over increasing horizons and perturbations. 
\textbf{(c)} The same reachability engine is integrated into sampling-based MPC, where candidate trajectories are evaluated using both task objectives and reachable-set penalties. Parallel reachability evaluation and gradient-based refinement enable uncertainty-aware online planning with certified reachable-set over-approximations.\vspace{-9pt}}\label{fig:method}
\end{figure*}

\vspace{-3pt}
\section{DiffReach: A GPU-Parallelized and Differentiable Reachability Engine}\label{sec:reach}
% \vspace{-2pt}

In this section, we present \textit{DiffReach}: a parallel, differentiable reachability pipeline for NN-in-the-loop systems (\Figref{fig:method}(a)). \Secref{subsec:repre} introduces a fixed-shape Taylor model (TM) representation that enables efficient batched computation. \Secref{subsec:ct_ana_dyn} develops TM flowpipe propagation for continuous-time (CT) analytical dynamics, while \Secref{subsec:ct_nn_dyn} incorporates neural components via CROWN through a unified TM--CROWN interface. Building on this interface, \Secref{subsec:ct_nn_ctl} presents closed-loop reachability for NN feedback systems by alternating controller certification and TM propagation. Finally, \Secref{subsec:dt_dyn} extends the framework to discrete-time (DT) systems via stepwise propagation. Overall, this unifies analytical and NN components under a shared affine representation and enables GPU-parallel reachability for learning and planning.

\subsection{Representation of Reachable Sets}\label{subsec:repre}

We represent reachable sets using \emph{Taylor models} (TMs), which provide a certified representation of nonlinear functions over bounded domains:
\begin{equation}
f(x) = p_k(x) \oplus I, \quad x \in X,
\end{equation}
\looseness-1where $p_k$ is a $k$-th order polynomial approximation and $I = [\underline{I}, \overline{I}]$ soundly bounds truncation and higher-order errors. This ``polynomial + interval'' form enables guaranteed set propagation, since bounding $p_k$ over $X$ and adding $I$ yields a sound enclosure of $f(X)$. Taylor models are also closed under common operations such as addition, multiplication, and composition, making them suitable for nonlinear reachability analysis.

\looseness-1Classical Taylor models use variable-size polynomial representations whose complexity depends on the number of monomials, which grows combinatorially with state dimension and polynomial order. While effective in CPU-based tools, these representations are not well suited for GPU execution, which prefer static tensor shapes and dense matrix operations.

\looseness-1To address this, we adopt a \emph{fixed-shape} TM representation by restricting the polynomial component to linear or quasi-quadratic forms. For an $n$-dim state, we represent each TM as
\begin{equation}
F(z) = c + Az \oplus I, \quad \text{or} \quad F(z) = c + Az + B\,\tau z \oplus I,
\end{equation}
where $z \in \mathbb{R}^{n+1}$ augments the state with time, $c \in \mathbb{R}^n$, $A, B \in \mathbb{R}^{n \times (n+1)}$, and $I \subset \mathbb{R}^n$ is the interval remainder. This yields a fixed $O(n^2)$ representation with static tensor shapes.

The affine term $Az$ preserves first-order correlations, while the quasi-quadratic term $B\,\tau z$ captures limited curvature without incurring the $O(n^3)$ cost of full quadratic parameterizations. As a result, TM operations such as composition and bounding reduce to batched matrix operations that are compatible with GPU acceleration and automatic differentiation.

\subsection{Propagation under Continuous-Time Analytical Dynamics}\label{subsec:ct_ana_dyn}

\looseness-1Given the TM representation above, we now describe reachable-set propagation for continuous-time analytical dynamics. Consider the uncontrolled system
$\dot{x}(\tau) = f_{\mathrm{ct,dyn}}(x(\tau)),$
with uncertain initial state $x(0) \in X_0$. The goal is to compute a sound enclosure of all states reachable over a finite horizon $t$.

We adopt \emph{flowpipe construction}, which propagates the initial set through a sequence of Taylor models. Over a small interval of length $h$, we compute a TM enclosure of the solution using truncated Picard iteration, summarized in~\Theoref{theo:tm_picard_step}.

\begin{theorem}[TM flowpipe step via Picard iteration~\cite{chen2012taylor}]\label{theo:tm_picard_step}
Consider the ODE \(\dot{x}(\tau)=\ctdyn(x(\tau))\) with initial set \(x(0)\in\gX_i\), and fix a step size \(h>0\). The $(i+1)$-th TM flowpipe over $\tau \in [0,h]$ is given by
\[
X^{(i+1)}(x_0, \tau)=p_k(x_0, \tau) \oplus I, \quad x_0\in\gX_i,
\]
where \(p_k\) is obtained via truncated $k$-order Picard iteration:
\begin{equation}\label{eq:picard}
\textstyle g_{j+1}(x_0,\tau) = \text{Trunc}_j\left(x_0 + \int_0^\tau \ctdyn\bigl(g_{j}(x_0,s)\bigr)\,ds\right),
\end{equation}
for $j=0,\dots,k-1$, with $g_0(x_0,\tau)$ initialized as the polynomial part of $X^{(i)}(x_0,h)$.

The interval remainder \(I\) is obtained via Picard-based validation and refinement. Starting from a conservative estimate \(I_0\), we replay the Picard iteration with
\[
g_0(x_0,\tau) = p_k(x_0,\tau) \oplus I_0.
\]
Let \(I_1 := g_1(x_0,\tau) - p_k(x_0,\tau)\) denote the induced remainder after one iteration. If the iteration is contractive, i.e., \(I_1 \subseteq I_0\), we continue refinement to obtain the minimal contractive interval \(I\) subject to termination conditions. Otherwise, \(I_0\) is enlarged until a contractive remainder is achieved.
\end{theorem}

Each step constructs a local polynomial approximation of the trajectory over $[0,h]$ while rigorously bounding truncation and higher-order errors through the interval remainder. Iterating this procedure over $N:=\ceil{\frac{t}{h}}$ steps yields a certified reachable tube over horizon $t$. In practice, repeated propagation can accumulate over-approximation error over long horizons. Classical reachability tools mitigate this effect through affine normalization, reparameterization, and wrapping control; we defer these details to Appendix~\ref{app:reach_ct_ana_dyn} and refer to~\cite{chen2015reachability,chen2016decomposed}.

\subsection{Neural Components via CROWN}\label{subsec:ct_nn_dyn}

The flowpipe construction above applies directly to analytical dynamics. However, when dynamics or controllers are represented by neural networks, directly propagating Taylor models (TMs) through the network is often either overly conservative for linear TMs or computationally prohibitive for higher-order TMs. Linear TM propagation corresponds to forward bound propagation through the network, where local relaxations are constructed independently at each nonlinear layer, leading to accumulated over-approximation. Higher-order TMs can improve tightness, but are difficult to implement efficiently with matrix operations and scale poorly to large networks.

In contrast, \emph{CROWN} (\Theoref{theo:crown}) computes affine bounds via backward bound propagation that accounts for downstream layers, often yielding significantly tighter bounds at moderate additional cost. We therefore leverage CROWN (through its JAX implementation in \texttt{jax\_verify}~\cite{jax_verify2020}) to certify neural components within the TM-based pipeline.

\begin{theorem}[CROWN linear bounds of NN~\cite{zhang2018efficient,xu2020automatic}]\label{theo:crown}
    Given an NN \(f_\text{NN}\) with input dimension \(n_i\) and output dimension \(n_o\), and a bounded input \(x\in\gX\) where \(\gX\) is a box set, there exist provable linear lower and upper bounds on \(f_\text{NN}(x)\) such that
    \begin{equation}
        \lowerW x + \lowerb \leq f_\text{NN}(x) \leq \upperW x + \upperb,
    \end{equation}
    where \(\lowerW, \upperW\in \sR^{n_o\times n_i}\) and \(\lowerb, \upperb\in \sR^{n_o}\).
\end{theorem}

A key observation is that a \emph{linear TM} \(p(z)\oplus I = c + Az \oplus I\) is equivalent to a pair of affine bounds with shared slope:
\[
\lowerW=\upperW=A,\quad \lowerb=c+\lowerI,\quad \upperb=c+\upperI.
\]
This allows analytical and neural components to share a common representation. To preserve this structure, we enforce a \emph{shared-slope constraint} in CROWN so that the lower and upper bounds use the same linear term, allowing neural network outputs to be represented again as linear TMs.

In our setting, neural networks are evaluated on set-valued TM inputs. Since standard CROWN assumes box-set inputs, we apply input reparameterization and certify an equivalent network, as summarized in \Theoref{theo:crown_tm}. We defer the proof to Appendix~\ref{app:reach_ct_nn_dyn}.

\begin{theorem}[CROWN bounds of NN with linear TM input]\label{theo:crown_tm}
    Given a neural network \(f_\text{NN}\) with input dimension \(n_i\) and output dimension \(n_o\), and a bounded input \(x\) parameterized by a linear TM \(g(z)=c_g+A_g z\oplus I_g\), where \(z\in\gZ\subset\sR^{n_z}\), \(A\in\sR^{n_i\times n_z}\), and \(I_g=[\lowerI_g,\upperI_g]\subset\sR^{n_i}\), we can construct a linear TM \(p(z)\oplus I = c+Pz \oplus I\) that soundly over-approximates \(f_\text{NN}(g(z))\), where \(c \in \sR^{n_o}, P\in\sR^{n_o\times n_z}\) and \(I=[\lowerI,\upperI]\subset\sR^{n_o}\). In particular, \(P\) and \(I\) can be obtained by applying CROWN (with the enforced constraint \(\lowerW=\upperW\)) to the equivalent network \(f_\text{NN}(c+Az+r)\) with inputs \(z\in\gZ\) and \(r\in I_g\).
\end{theorem}

\begin{remark}[Quasi-quadratic TM inputs]
For quasi-quadratic TMs, we bound the higher-order term as an interval and reduce the input to a linear TM before applying CROWN. Since the step size \(h\) is small, the resulting conservatism is typically modest in practice.
\end{remark}

CROWN therefore serves as a drop-in module for neural components within the TM-based reachability pipeline. Analytical dynamics are propagated with TM arithmetic (\Secref{subsec:ct_ana_dyn}), while neural components are certified through CROWN under the same linear-TM interface. This unified representation avoids intermediate interval conversion, preserves cross-state correlations, and maintains end-to-end differentiability.

\subsection{Closed-Loop Reachability for Continuous-Time Systems}\label{subsec:ct_nn_ctl}

\looseness-1We now consider reachability of continuous-time systems under neural feedback control. Our key idea is to compute reachable sets of the closed-loop system by \emph{alternating} between (i) certifying the neural controller using CROWN and (ii) propagating the system dynamics using TM flowpipes under zero-order hold. This yields a unified pipeline in which both controller and dynamics share the same TM-based representation.

We consider a continuous-time control system
\(\dot{x}(\tau) = f_{\mathrm{ct,dyn}}(x(\tau), u_{\mathrm{ctl}}),\)
where the control input is given by a neural controller
\(u_{\mathrm{ctl}} = f_{\mathrm{ct,ctl}}(x(i\delta))\)
and held constant over each control interval $\tau \in [i\delta,(i+1)\delta]$ (zero-order hold). The initial state satisfies $x(0)\in X_0$. For clarity, we omit additional controller inputs such as reference trajectories \((\refx,\refu)\), which can be treated as known exogenous inputs to \(\ctctl\).

To compute a certified reachable tube over a finite horizon, we decompose each control interval into two steps:

\textbf{(A) Controller certification.}
At $\tau=i\delta$, we apply the TM--CROWN interface from~\Secref{subsec:ct_nn_dyn} to the TM enclosure of \(x(i\delta)\) and obtain a certified linear-TM enclosure of the controller output
\(u_{\mathrm{ctl}} = f_{\mathrm{ct,ctl}}(x(i\delta)).\)

\textbf{(B) Dynamics propagation.}
Given the certified control enclosure, we propagate the dynamics over $\tau \in [i\delta,(i+1)\delta]$ using the TM flowpipe method from~\Secref{subsec:ct_ana_dyn}. The control input is treated as constant by augmenting the state with \(u_{\mathrm{ctl}}\) and enforcing \(\dot{u}_{\mathrm{ctl}}=0\). Each control interval is discretized into smaller steps of size $h \ll \delta$ to obtain a certified enclosure of the reachable tube and the next boundary state \(x((i+1)\delta)\).

Iterating these two steps over all control intervals yields a sound over-approximation of the reachable set of the closed-loop system. \Algref{alg:reach_ct_nn_ctl} summarizes the alternating procedure. The outer loop performs controller certification at frequency \(1/\delta\), while the inner loop applies TM flowpipe propagation with atomic step size \(h\). Reachable sets are represented as TM-valued states \(X(\tau)\), where evaluating a TM at a time point yields a box enclosure. The operation \(\mtt{Ctl_{CROWN}}()\) implements controller certification via the TM--CROWN interface, while \(\mtt{ReachStep}()\) performs one TM flowpipe step under the augmented zero-order-hold dynamics.

Closely related templates combining TM-based flowpipe propagation with CROWN-style certification have appeared in prior work~\cite{huang2022polar, CROWNReach2026}. However, these approaches typically connect the two modules through interval-only interfaces, e.g., converting TMs to intervals before applying CROWN. This discards affine dependencies captured by TM representations and can amplify conservatism over long horizons. In contrast, our method maintains a unified linear-TM representation throughout the pipeline, avoiding intermediate intervalization while preserving dependencies across states and time steps. This unified interface further enables the fully neural setting where both \(\ctdyn\) and \(\ctctl\) are neural networks.

\setlength{\floatsep}{0pt}       % space between two floats
\setlength{\textfloatsep}{0pt}   % space between floats and text
\setlength{\intextsep}{0pt}      % space around in-text floats
\begin{algorithm}[t]
    \caption{Reachability of CT NN Control Systems. }%\textcolor{brown}{Comments} are in brown.
    \label{alg:reach_ct_nn_ctl}
    {\small
    \begin{algorithmic}[1]
        \State \looseness-1\textbf{Inputs:} dynamics \(\ctdyn(x,u)\) (analytical or NN), NN controller \(\ctctl\), initial box \(\gX_0=[\lowerx_0,\upperx_0]\), total control steps \(N_\text{ctl}\), control interval \(\delta=Kh\), atomic step interval \(h\), TM-flowpipe params \(\theta_{\mathrm{fp}}
\)
        \State \textbf{Outputs:} certified reachable set \(\{\gX_j\}_{j=0}^{N} (N=N_\text{ctl}K)\) 
        \vspace{0.5em}

        \State \textcolor{brown}{\# Precompute step-local boxes and initialize TM state}
        \State \(X_\text{pre} \gets [\mtt{BuildLinearTM}(\gX_0),\;\mathbf{0}]\) \Comment{\textcolor{brown}{affine TM for \(\tilex\)}}

        \State \(\tilde{f} \gets [\ctctl,\;\mathbf{0}]\) \Comment{\textcolor{brown}{define augmented dynamics $\tilde{f}(x,u)$}}

        \vspace{0.5em}
        \State \textcolor{brown}{\# Outer loop over control updates (zero-order hold)}
        \For{$i=0,\dots,N_\text{ctl}-1$}
            \State \textcolor{brown}{\# (A) Impose NN control at the boundary using CROWN--TM}
            \State \(X_{iK} \gets X_\text{pre}(h)\) \Comment{\textcolor{brown}{TM at the control boundary $i\delta$}}
            \State \(X_{iK}  \gets \mtt{Ctl_{CROWN}}\big(X_{iK} ,\ctctl\big)\)
            \Comment{\textcolor{brown}{get affine TM \(u=\ctctl(x)\)}}

            \vspace{0.3em}
            \State \textcolor{brown}{\# (B) Inner loop: propagate augmented TM under \(\tilde{f}\) for $\delta$}
            \State  \(X_\text{in} = X_{iK}\)
            \For{$j=1,\dots,K$}
                \State \((X_\text{in},\;\gX_{iK+j}) \gets \mtt{ReachStep}\big(\tilde{f},X_\text{in};h,\theta_{\mathrm{fp}}
\big)\)
            \EndFor
            \State \(X_\text{pre} \gets X_\text{in}\) \Comment{\textcolor{brown}{TM for next control step}}
        \EndFor

        \vspace{0.3em}
        \State \Return \(\{\gX_j\}_{j=0}^{N}\)
    \end{algorithmic}
    }
\end{algorithm}

\vspace{-4pt}
\subsection{Discrete-Time Systems}\label{subsec:dt_dyn}

We now consider discrete-time systems, where reachability reduces to stepwise propagation without flowpipe construction. Consider the system
\(x_{k+1} = f_{\mathrm{dt,dyn}}(x_k, u_k),\)
with uncertain initial state $x_0 \in X_0$ and planned or exogenous inputs \(\{u_k\}_{k=0}^{H-1}\). Our goal is to compute a sequence of reachable sets \(\{X_k\}\) over a finite horizon.

At each time step, we apply the TM--CROWN interface (\Secref{subsec:ct_nn_dyn}) to certify the one-step map and obtain an affine TM enclosure of the next reachable set:
\[
X_{k+1} = f_{\mathrm{dt,dyn}}(X_k, u_k).
\]
All operations are performed within the same linear-TM representation. Algorithm~\ref{alg:reach_dt_dyn} summarizes this stepwise procedure, where each iteration applies CROWN-based certification followed by TM evaluation.

Compared to the continuous-time setting, this formulation removes inner-loop integration and yields a simpler, more efficient propagation scheme. It is particularly suitable for planning and control tasks requiring batched rollout over finite horizons, and integrates naturally with reachability-aware MPC and trajectory optimization. The same TM--CROWN interface can also support discrete-time neural feedback controllers analogously to the continuous-time case.

\begin{algorithm}[t]
    \caption{Reachability of DT Dynamics.}
    \label{alg:reach_dt_dyn}
    {\small
    \begin{algorithmic}[1]
        \State \textbf{Inputs:} DT model \(\dtdyn\), initial box \(\gX_0=[\lowerx_0,\upperx_0]\), planned inputs \(\{u_k\}_{k=0}^{H-1}\), horizon \(H\)
        \State \textbf{Outputs:} reachable sets \(\{\gX_t\}_{t=0}^{H}\)
        \vspace{0.5em}

        \State \(X_0 \gets \mtt{BuildLinearTM}(\gX_0)\)
        \For{$k=0,\dots,H-1$}
            \State \textcolor{brown}{\# Certify one-step map \(x_{k+1}=\dtdyn(x_k, u_k)\)}
            \State \(X_{k+1} \gets \mtt{CROWN}\big(\lambda x.\,\dtdyn(x,u_k),\;X_k\big)\) \Comment{\textcolor{brown}{affine bounds}}
            \State \(\gX_{k+1} \gets \mtt{EvalInt}\big(X_{k+1}\big)\)
            \Comment{\textcolor{brown}{box enclosure of next state}}
        \EndFor
        \State \Return \(\{\gX_k\}_{k=0}^{H}\)
    \end{algorithmic}
    }
\end{algorithm}
\vspace{-3pt}
\section{DiffReach-Robotics: Reachability-Aware \\ Model Learning and Planning}

Our reachability engine, DiffReach (\Secref{sec:reach}), provides differentiable reachable-set objectives that can be embedded into downstream optimization. We use this primitive to regularize the training of NN components in~\Secref{subsec:training}, and to incorporate reachable-set penalties into sampling-based MPC in~\Secref{subsec:planning}. We refer to these robotics-focused methods collectively as \textit{DiffReach-Robotics} (\Figref{fig:method}(b,c)).

\vspace{-2pt}
\subsection{Certified training of neural dynamics and controllers}\label{subsec:training}

We train neural dynamics and controllers with standard prediction or tracking losses augmented by reachability regularization: for rollouts under bounded initial uncertainty, we penalize the growth of certified reachable sets computed by the corresponding DT or CT reachability engine.

\noindent\textbf{Training DT neural dynamics.\quad}
We learn a discrete-time neural dynamics model \(\dtdyn\) that is accurate over multi-step rollouts and yields compact certified reachable sets. From random-interaction data
\(
\gD=\{(x_t^m,u_t^m)\mid t=0,\dots,T-1,\; m=0,\dots,M-1\},
\)
we form \(M'\) episodes of length \(T_h+1\) and train \(\dtdyn\) with the autoregressive prediction loss
\begin{equation}\label{eq:pred_loss_dt_dyn}
\gL_{\text{pred}}
=\textstyle \frac{1}{M'T_h}\sum_{m=0}^{M'-1}\sum_{t=0}^{T_h-1}
w_t\bigl\|x_{t+1}^m-\hx_{t+1}^m\bigr\|_2^2,
\end{equation}
where \(\hx_{t+1}^m=\dtdyn(\hx_t^m,u_t^m)\), \(\hx_0^m=x_0^m\), and time-increasing weights \(\{w_t\}\) emphasize long-horizon consistency.

\emph{Reachability-aware loss.}
Prediction accuracy alone does not control uncertainty growth, so we regularize the DT reachable tube from \(\gX_0^m=\gB_{\epsilon}(x_0^m)\) under actions \(u_{0:T_h-1}^m\). Let
\(\{\gR_t(u_{0:T_h-1}^m,\gX_0^m)\}_{t=1}^{T_h}=\{\gX_t^m\}_{t=1}^{T_h}\)
denote the certified reachable sets computed by~\Secref{subsec:dt_dyn}. We define
\begin{equation}\label{eq:reach_loss_dt_dyn}
\gL_{\text{reach}}
= \textstyle\frac{1}{M'}\sum_{m=0}^{M'-1}\log\!\bigl(1+V_{\gR,\epsilon}^m\bigr),
\end{equation}
where
\(
V_{\gR,\epsilon}^m=\sum_{t=1}^{T_h}\sum_{j=0}^{n-1}\mtt{width}(\gX_{t,j}^m)
\)
is a tube-volume proxy. The width sum and \(\log\) term stabilize gradients when volumes vary widely. The final objective is \(\gL_{\text{total}}=\gL_{\text{pred}}+\lambda\,\gL_{\text{reach}}\), with regularization weight \(\lambda\).

\emph{Curriculum learning with horizon and perturbation scheduling.\quad}
Joint optimization can become unstable for large rollout horizons \(T_h\) and perturbation radii \(\epsilon\), since early training may produce excessively large reachable sets and poorly conditioned reachability losses. To mitigate this, we use a curriculum that progressively increases \(T_h\) from one-step prediction to the target horizon using a logarithmic schedule, while simultaneously decreasing \(\epsilon\) from a larger initial radius to the desired value. This stabilizes early training while encouraging reachability-friendly dynamics over long horizons. The full procedure is given in \Algref{alg:train_dt_dyn} in Appendix~\ref{app:training}.

\noindent\textbf{Training CT neural controllers.\quad}
\looseness-1We train a CT neural controller \(\ctctl\) to track reference signals while producing compact certified reachable tubes. We collect a dataset
\(
\gD=\{(x_t^m,u_t^m,y_{\mathrm{ref},t}^m)\mid t=0,\dots,T-1,\; m=0,\dots,M-1\}
\)
at control frequency \(q_{\text{ctl}}=1/\delta\) by executing a nominal controller under randomized references, and reformat it into \(M'\) episodes of length \(T_t+1\). Using a differentiable CT dynamics model \(\ctdyn\) (analytical ODE or neural ODE) as a rollout proxy, we train \(\ctctl\) with
\begin{equation}\label{eq:track_loss_ct_ctl}
\gL_{\text{track}}
= \frac{1}{M'T_t}\sum_{m=0}^{M'-1}\sum_{t=0}^{T_t-1}
w_t\Bigl(\|u_t^m-\hu_t^m\|_2^2 + \gamma \|x_{t+1}^m-\hx_{t+1}^m\|_2^2\Bigr),
\end{equation}
where \(\hu_t^m=\ctctl(\hx_t^m,y_{\mathrm{ref},t}^m)\) is held constant over \(\tau\in[t\delta,(t+1)\delta]\), and \(\hx_{t+1}^m\) is obtained by integrating
\(
\dot{\hx}^m(\tau)=\ctdyn(\hx^m(\tau),\hu_t^m)
\)
using the JAX-based ODE solver \texttt{diffrax}~\cite{kidger2021on}. The second term encourages consistency between predicted actions and their induced state transitions.

Analogously to DT dynamics training, we regularize \(\ctctl\) using the CT NN-control reachability engine (\Secref{subsec:ct_nn_ctl}) and the same reachable-tube loss \(\gL_{\text{reach}}\) in~\eqref{eq:reach_loss_dt_dyn}, initialized from \(\gB_{\epsilon}(x_0^m)\). We use the same curriculum strategy to stabilize training and obtain controllers that are both performant and reachability-friendly. The same procedure also applies to CT dynamics training.

\subsection{Reachability-aware MPC}\label{subsec:planning}

\looseness-1Using the differentiable reachability engine and reachability-regularized neural models, we develop a reachability-aware MPC that jointly optimizes task cost and reachability penalties.

\noindent \textbf{Planning.\quad}
\looseness-1Solving~\eqref{eq:trajop_unc} with neural dynamics in real time is challenging, as
direct optimization and classical reachability-based methods are computationally expensive, while sampling-based planners such as MPPI~\citep{williams2017model} and CEM~\citep{rubinstein2013cross} must evaluate many candidate trajectories under uncertainty. We integrate the DT reachability engine (\Secref{subsec:dt_dyn}) directly into this process: for each sampled action sequence, we roll out the nominal dynamics, compute the corresponding reachable sets, and evaluate
\begin{equation}\label{eq:plan_obj}
    \hspace{-5pt}O = \textstyle\sum_{t=t_0}^{t_0+H} c(\hx_t, u_t) + C\max (0, - G\!\left(\gR_t(\vu,\gB_{\epsilon}(\hx_{t_0}))\right))
\end{equation}
where \(\max(0,-G(\cdot))\) penalizes unsafe trajectories or large reachable sets, and \(C\) controls the penalty strength. Parallelization enables efficient evaluation across many samples, while gradient-based refinement is applied only to the top candidate at the final iteration to further reduce reachable-set size.

\noindent \textbf{MPC.\quad} 
The planner is embedded in a receding-horizon MPC loop with online replanning. After each update, the first few planned actions and predicted states are passed as references to the neural controller \(\ctctl\) for low-level execution under \(\ctdyn\). Guarantees are with respect to the learned dynamics; in practice, calibrated model-error bounds (e.g., via conformal prediction) can be incorporated into the reachable sets.

\section{Results}

\looseness-1We evaluate the proposed framework along four directions: (i) runtime and tightness of the reachability engine against DT and CT baselines (Sec.~\ref{sec:results_baselines}); (ii) scalability and conservativeness on higher-dimensional systems (Sec.~\ref{sec:results_scalability}); (iii) certified training of neural dynamics and controllers (Sec.~\ref{sec:results_certified}); and (iv) robustness of reachability-aware MPC in simulation and hardware (Sec.~\ref{sec:results_planning}). All experiments were run on an RTX 4090 GPU.

\subsection{Reachability Tool Evaluations}\label{sec:results_baselines}

\begin{figure}
    \centering
    \includegraphics[width=\linewidth]{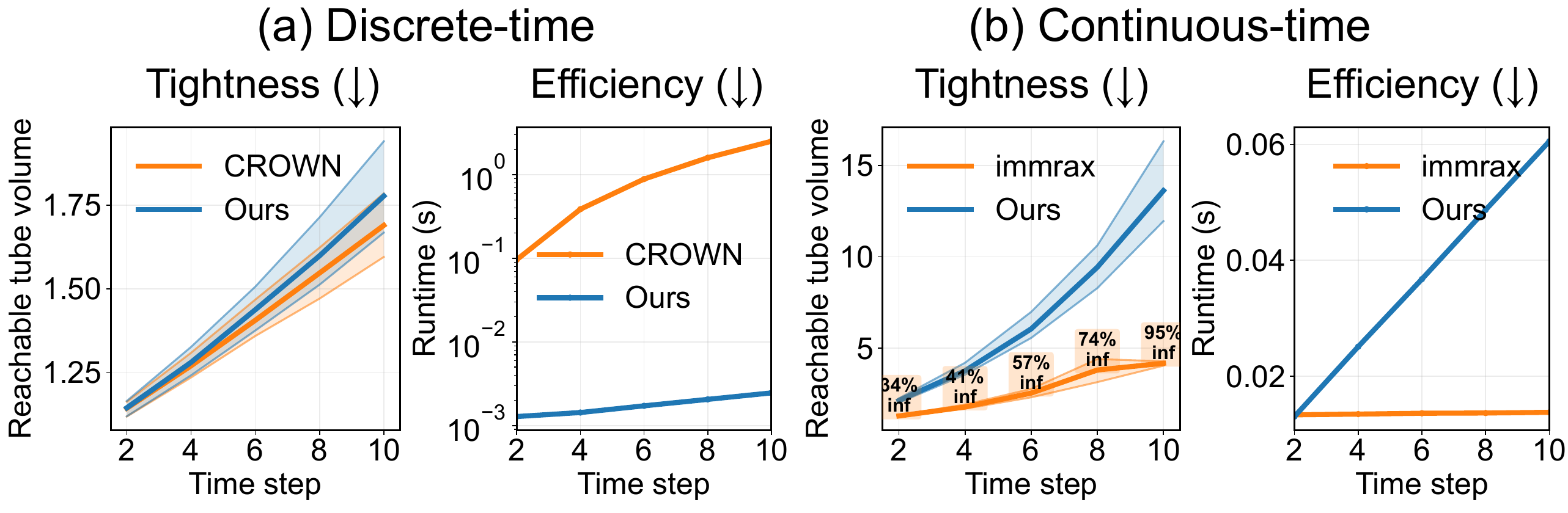}\vspace{-6pt}
    \caption{\looseness-1\textbf{Benchmarking our JAX reachability tool.} \textbf{Top}: Compared with DT CROWN reachability, our method achieves similar reachable-set volumes while being $\approx$100$\times$ faster. \textbf{Bottom}: Compared with CT immrax reachability, our method produces substantially tighter reachable sets with comparable runtime.}
    \label{fig:baseline}
\end{figure}

We benchmark the reachability engine against state-of-the-art DT (CROWN~\cite{zhang2018efficient}) and CT (immrax~\cite{harapanahalli2024immrax}) baselines with GPU support in JAX in terms of runtime and reachable-set conservativeness (Fig.~\ref{fig:baseline}), before integrating the method into downstream learning and planning.

\looseness-1We perform reachability over a horizon of 10 for 128 randomly sampled $\mathcal{X}_0$. In the DT setting (Fig.~\ref{fig:baseline}, top), we use a randomly initialized neural dynamics model $\dtdyn$ with three 96-unit hidden layers. CROWN produces slightly tighter tubes because it propagates bounds jointly across time steps, unlike our truncated propagation (Alg.~\ref{alg:train_dt_dyn}). However, our JAX implementation is $\approx$100$\times$ faster while maintaining comparable tightness.

\looseness-1For the CT setting (Fig.~\ref{fig:baseline}, bottom), we evaluate analytical quadrotor dynamics. immrax is slightly faster due to lightweight interval propagation, but produces substantially looser reachable sets and frequently returns infinite-volume tubes (up to 95\%) from numerical overflow. Even after partitioning $\mathcal{X}_0$ into 64 subsets and taking the union of the resulting reachable sets, conservativeness remains high. We therefore report average tube volume over successful trials, and separately report the failure rate (``inf'') for immrax. Overall, our method achieves comparable runtime with substantially tighter over-approximations.

Beyond forward reachability, the differentiable GPU-parallel implementation also enables practical refinement strategies for reducing conservativeness (Table~\ref{tab:refine}). Parallelized input splitting significantly reduces reachable-set volume with modest runtime increase, while gradient-based refinement further improves tightness without modifying the underlying algorithm. These results demonstrate that differentiable reachability can serve not only as a verification tool, but also as an optimization primitive for downstream learning and planning tasks.
\vspace{5pt}
\begin{table}[h]
\centering
\scriptsize
\setlength{\tabcolsep}{4pt}

\begin{minipage}[t]{0.48\linewidth}
\centering
\caption{Refinement strategies.}
\label{tab:refine}
\vspace{-8pt}
\begin{tabular}{@{}lcc@{}}
\toprule
Method & Volume & Runtime (s) \\
\midrule
Baseline & 91.94 & 0.0021 \\
Input splitting (8) & 27.56 & 0.0048 \\
Gradient (20 iters) & 18.18 & 0.1084 \\
\bottomrule
\end{tabular}
\end{minipage}
\hfill
\begin{minipage}[t]{0.48\linewidth}
\centering
\caption{Runtime comparison.}
\label{tab:multi_quad}
\vspace{-8pt}
\begin{tabular}{@{}lcc@{}}
\toprule
Method & 1 Quad (12D) & 6 Quad (72D) \\
\midrule
Sampling & 0.2601 & 2.2455 \\
Ours & 0.1399 & 0.1976 \\
\bottomrule
\end{tabular}
\end{minipage}
\end{table}
\vspace{-8pt}

\subsection{Scalability on Higher-Dimensional Systems}\label{sec:results_scalability}

We further evaluate scalability and reachable-set conservativeness on moderately higher-dimensional robotic systems beyond the base T-pushing and single-quadrotor tasks.

% \begin{wrapfigure}{r}{0.3\textwidth}
%     \centering
%     \includegraphics[width=\linewidth]{figs/rebuttal/multi_quad.pdf}
%     \vspace{-20pt}
%     \caption{Coupled 6-quadrotor system.}
%     \label{fig:multi_quad}
% \end{wrapfigure}

We first evaluate a coupled 6-quadrotor swarm with \emph{72D state} and \emph{18D action}, where quadrotors move with spring-like coupling to the swarm center (Fig.~\ref{fig:multi_quad}). We apply $0.05$ perturbations to position and velocity and perform reachability over 100 steps with $\Delta t=0.02$. Compared with dense sampling (100k samples, under-approximation), our method produces reasonable over-approximations with average $1.37\times$ expansion on position dimensions while remaining efficient. As shown in Table~\ref{tab:multi_quad}, runtime remains close to the single-quadrotor case despite the substantially larger state dimension. Although memory scales as $\mathcal{O}(n^2)$ under the affine representation, runtime scales more favorably in practice due to GPU-parallel matrix operations.

\begin{figure}[h]
    \centering
    \includegraphics[width=\linewidth]{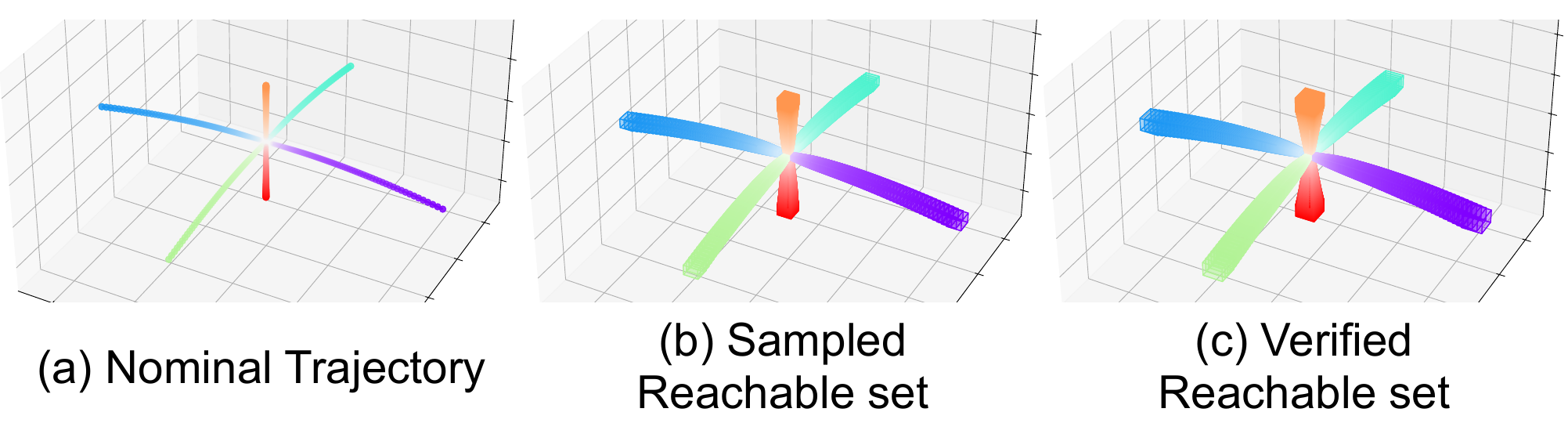}
    \vspace{-20pt}
    \caption{\textbf{Reachability on a coupled 6-quadrotor swarm}. \textbf{(a)} Nominal trajectory. Quadrotors fly toward 6 different directions while coupled by spring-like attraction to the swarm center. \textbf{(b)} Dense sampled reachable set (under-approximation) under uncertainty. \textbf{(c)} Certified reachable set computed by our method. Our method produces a reasonable certified over-approximation.}
    \label{fig:multi_quad}
\end{figure}
\vspace{8pt}

We also evaluate an acceleration-controlled 10-joint planar arm with \emph{40D state} and \emph{10D action}. We apply $\pm0.1$ rad/s perturbations to joint velocities and perform reachability over 100 steps with $\Delta t=0.01$. Table~\ref{tab:arm} compares against dense sampling (100k samples). Without splitting, reachable sets become noticeably more conservative on this articulated system. However, parallelized splitting substantially improves end-effector reachable widths while maintaining practical runtime. Overall, these results suggest that the framework can scale to moderately high-dimensional robotic systems, although tightness increasingly depends on system structure and refinement strategies such as splitting.

\vspace{5pt}
\begin{table}[h]
\centering
\scriptsize
\caption{Reachability on a 10-joint arm.}
\label{tab:arm}
\vspace{-6pt}
\begin{tabular}{lccc}
\toprule
Method & EEF X width & EEF Y width & Runtime (s) \\
\midrule
Sampling (100k) & 0.4799 & 0.6256 & 1.0944 \\
Ours (no split) & 0.8414 & 1.2264 & 0.0531 \\
Ours (64 splits) & 0.8118 & 0.9756 & 0.1186 \\
Ours (512 splits) & 0.7353 & 0.8999 & 0.3854 \\
\bottomrule
\end{tabular}
\end{table}
% \vspace{-5pt}

\subsection{Certified Training}\label{sec:results_certified}

\begin{figure*}
    \centering
    \includegraphics[width=\linewidth]{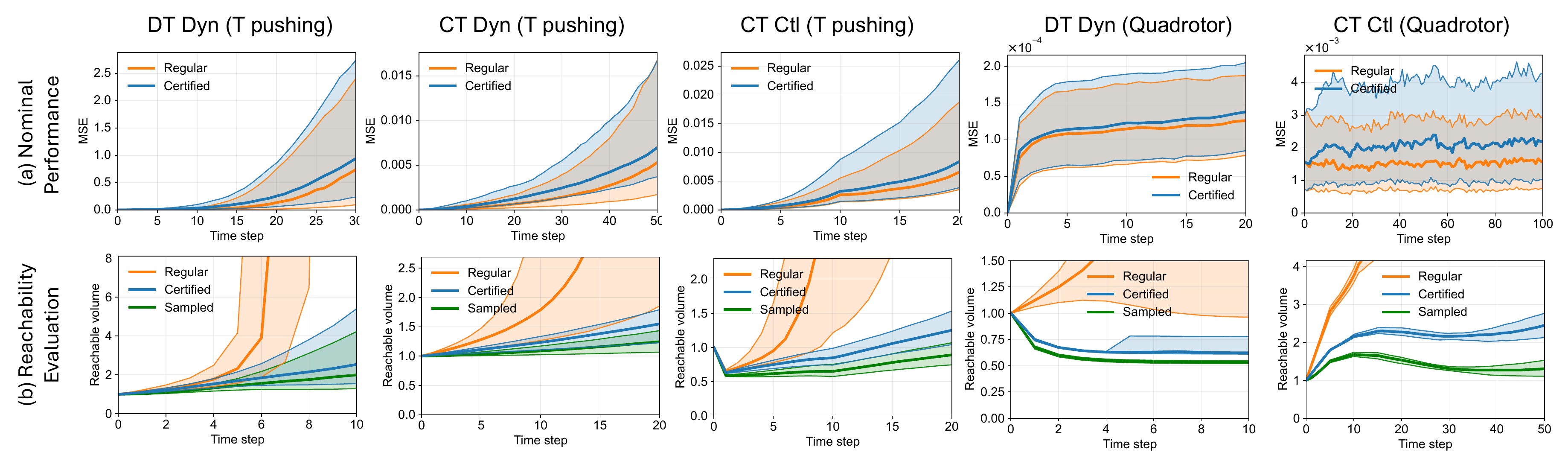}\vspace{-12pt}
    \caption{\looseness-1\textbf{Benchmarking certified training.} \textbf{Top}: Across CT/DT systems and tasks, certified training maintains nominal performance comparable to na\"ive training. \textbf{Bottom}: Certified training produces substantially tighter reachable sets that more closely match sampled lower-bound volumes.}\vspace{-15pt}
    \label{fig:certified_training}
\end{figure*}

\begin{figure}[h]
    \centering
    \includegraphics[width=\linewidth]{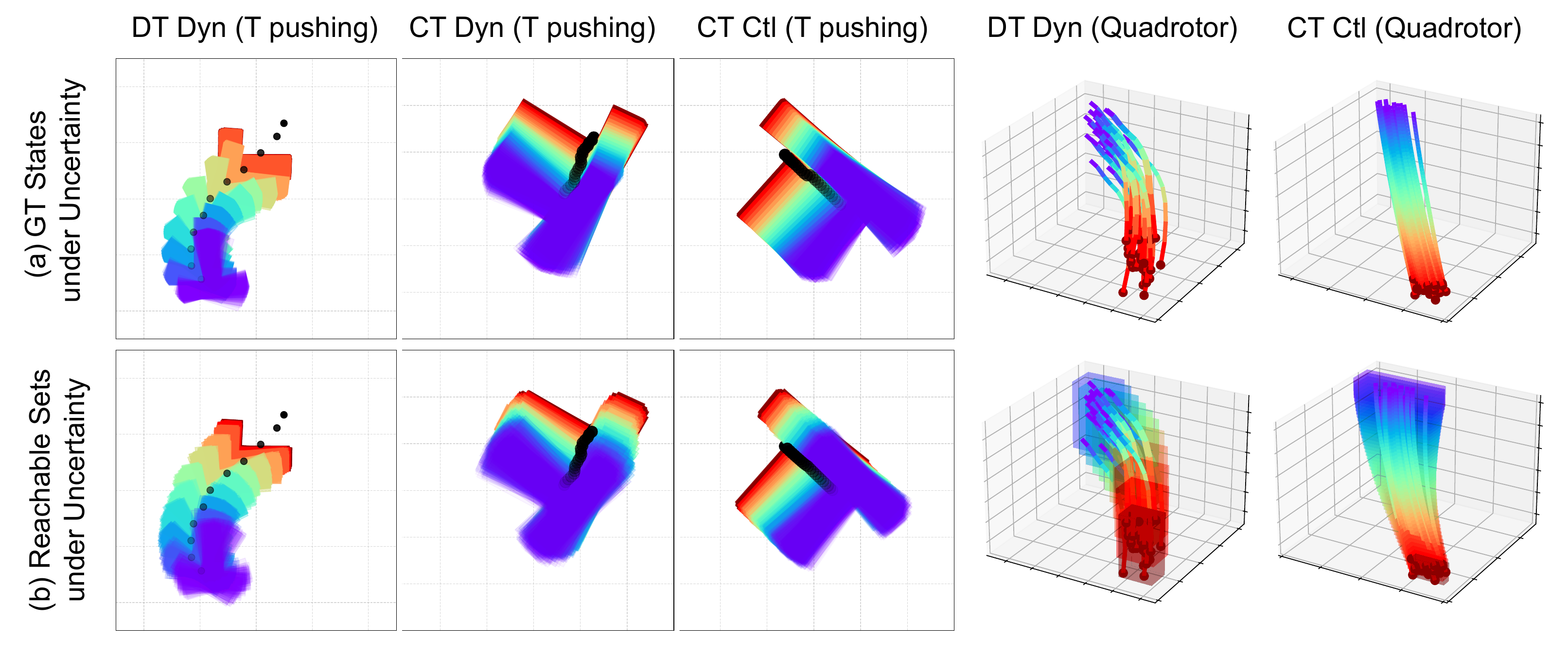}\vspace{-6pt}
    \caption{\looseness-1\textbf{Certified training visualizations.} \textbf{Top}: Rollouts from sampled initial states in $\mathcal{X}_0$ under the true dynamics or controller. \textbf{Bottom}: States sampled from the certified reachable sets closely overlap with the realizable trajectories.\vspace{9pt}}
    \label{fig:certified_training_model_viz}
\end{figure}

We evaluate certified training of CT and DT dynamics models and controllers on T-pushing and quadrotor tasks (Fig.~\ref{fig:certified_training}). Using 1,000 random initial sets $\mathcal{X}_0$ and test trajectories, we measure reachable-set tightness together with nominal prediction or tracking performance. Certified training achieves performance comparable to standard training with only slight degradation, while producing substantially tighter reachable sets. In contrast, models trained without reachability regularization yield significantly looser and less verifiable predictions.

To assess conservativeness, we estimate a lower bound on reachable-set volume (green, ``sampled") by propagating 64 sampled initial states from each $\mathcal{X}_0$ through the true dynamics and measuring the axis-aligned bounding-box volume at each step. Across tasks, certified training remains substantially closer to these sampled lower-bound trends than regular training. Fig.~\ref{fig:certified_training_model_viz} further illustrates this behavior: realized trajectories closely align with the certified reachable regions, demonstrating the fidelity of the resulting certificates.

We also study long-horizon reachable-set growth on the quadrotor task. Table~\ref{tab:long_horizon} reports the median reachable-volume ratio against dense sampling on valid trajectories with finite reachable-set volumes for DT Dyn and CT Ctl rollouts up to $5\times$ longer than those in Fig.~\ref{fig:certified_training}. DT reachability remains relatively stable under stepwise propagation, while CT conservativeness grows substantially over long horizons due to trajectories with infinite-volume reachable sets. Parallelized splitting on critical roll-pitch-yaw states improves tightness and reduces the fraction of unstable CT trajectories from 42\% to 24\%. Since the reported ratios are computed only over valid trajectories, intermediate horizons may not monotonically reflect this improvement. Nevertheless, splitting substantially improves long-horizon behavior, reducing the $5\times$ horizon ratio from $18.6154$ to $6.2407$. These results suggest that although long-horizon CT reachability remains challenging, splitting provides a practical mechanism for mitigating conservativeness growth.

\vspace{5pt}
\begin{table}[h]
\centering
\scriptsize
\caption{Tightness over long horizons on quadrotor.}
\label{tab:long_horizon}
\vspace{-6pt}
\begin{tabular}{lccccc}
\toprule
Horizon & $1 \times$ & $2 \times$ & $3 \times$ & $4 \times$ & $5 \times$ \\
\midrule
DT Dyn (no split) & 1.3221 & 1.3227 & 1.8782 & 1.8872 & 1.8898 \\
CT Ctl (no split) & 1.8648 & 2.1591 & 3.8261 & 3.8737 & 18.6154 \\
CT Ctl (8 splits, rpy) & 1.9171 & 2.9817 & 3.8261 & 3.0396 & 6.2407 \\
\bottomrule\vspace{-1pt}
\end{tabular}
\end{table}
\vspace{-5pt}

% Overall, certified training substantially improves verifiability while maintaining strong nominal performance across both CT and DT systems.

\subsection{Reachability-Aware MPC}\label{sec:results_planning}

\begin{figure}[h]
    \centering
    \includegraphics[width=\linewidth]{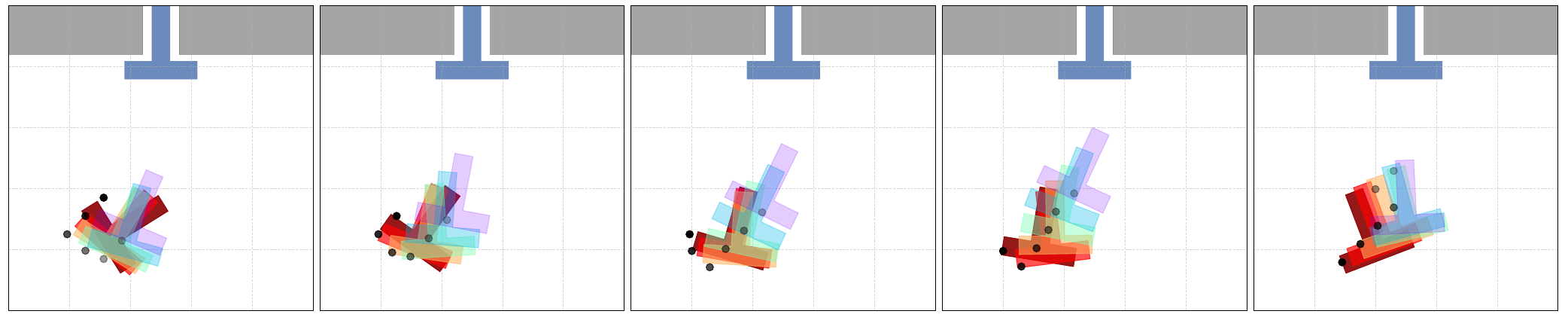}\vspace{-6pt}
    \caption{\looseness-1\textbf{Failure case of vanilla MPC with non-certified dynamics.} The planner places the pusher too close to the T-shaped object, leading to unintended contact and failure to reach the goal.}
    \label{fig:planning_failure}
\end{figure}

\begin{figure}[t]
    \centering
    \includegraphics[width=\linewidth]{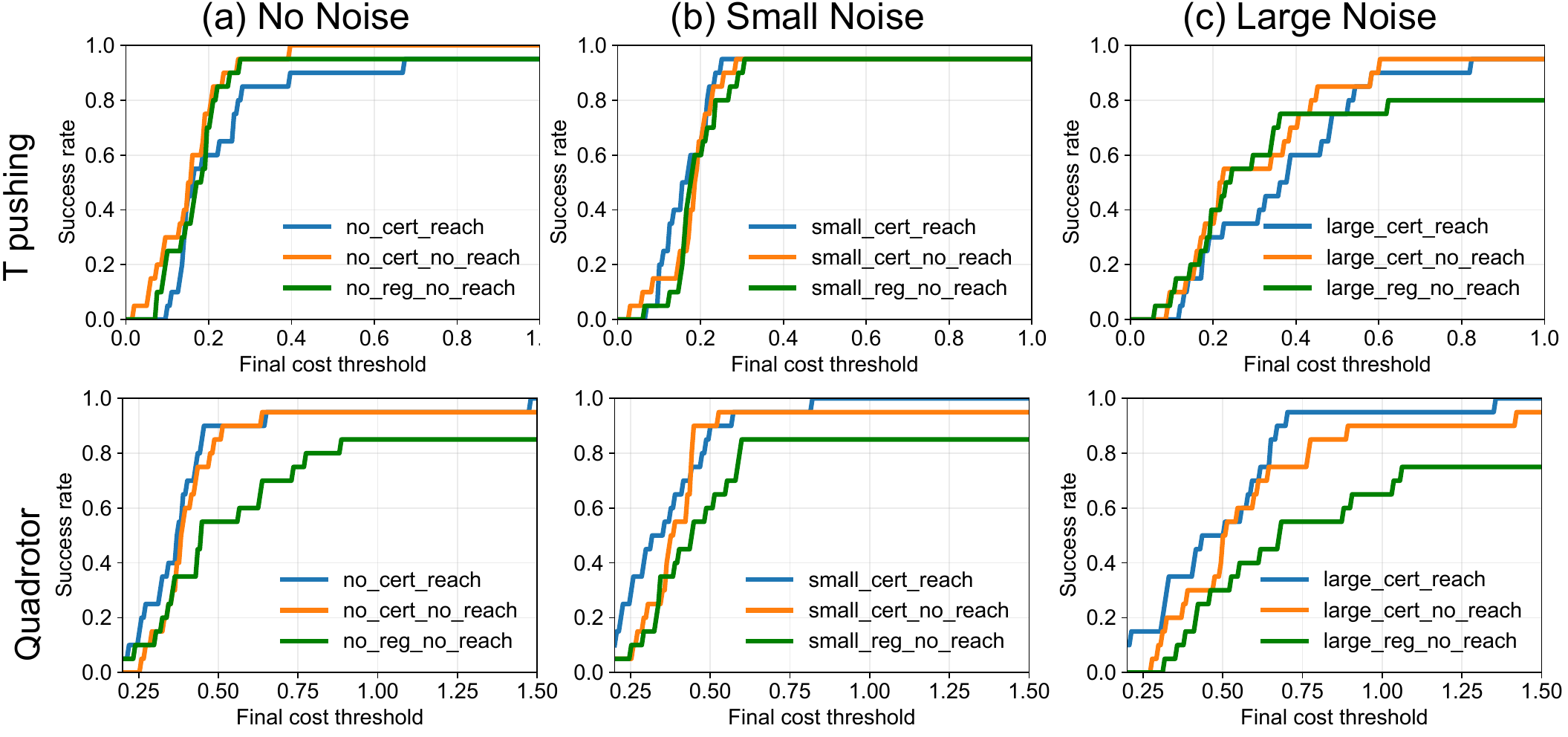}\vspace{-6pt}
    \caption{\looseness-1\textbf{Reachability-aware planning and MPC.} MPC performance under \textbf{(a)} no, \textbf{(b)} small, and \textbf{(c)} large disturbances. We compare \textbf{1.} reachability-aware planning with certified models (blue), \textbf{2.} vanilla planning with certified models (orange), and \textbf{3.} vanilla planning with regular models (green). Performance is similar under small disturbances, while under large disturbances the regular-model baseline degrades substantially, especially for quadrotor control.}\vspace{2pt}
    \label{fig:reachability_planning}
\end{figure}

\looseness-1Finally, we evaluate reachability-aware MPC under dynamics and control disturbances. For T-pushing, we apply bounded perturbations of $0.05$ and $0.015$ at the planning and low-level control layers; for quadrotor control, the corresponding bounds are $0.3$ and $0.1$. In Fig.~\ref{fig:planning}, MPC plans toward goal states while enforcing state constraints with tight reachable sets under uncertainty. The T-shaped object must avoid the gray unsafe region, while the quadrotor must avoid a spherical obstacle.

\looseness-1Fig.~\ref{fig:planning} visualizes MPC trajectories generated using the learned dynamics model and~\eqref{eq:trajop_unc}. For T-pushing, we sample configurations from the reachable sets to illustrate uncertainty propagation; for quadrotor navigation, we visualize reachable tubes projected into position space. Across both tasks, the planner satisfies constraints while remaining robust to bounded disturbances.

We further compare against vanilla MPC without reachability constraints using non-certified dynamics (Fig.~\ref{fig:planning_failure}). Here, the planner positions the pusher too close to the T-shaped object, and unintended contact perturbs the trajectory beyond recovery. This example highlights the importance of incorporating reachable-set information during planning rather than relying only on nominal predictions.

\looseness-1Quantitative comparisons are reported in Fig.~\ref{fig:reachability_planning}. Under no or small disturbances, all methods achieve similar performance. Under larger disturbances, however, reachability-aware planning with certified models consistently achieves higher success rates, while vanilla MPC with regular models degrades substantially, especially on the quadrotor task. These results suggest that reachability-aware MPC improves robustness without excessive conservativeness in nominal settings.

\looseness-1We also demonstrate a hardware rollout of reachability-aware MPC on the T-pushing task using a Kuka iiwa platform (Fig.~\ref{fig:planning}, middle row). The planner runs online at approximately $10$\,Hz ($5$\,Hz with gradient-based refinement), enabling real-time planning and smooth execution. While the current hardware evaluation is qualitative, it demonstrates the practical feasibility of integrating differentiable reachability into online robotic control.

\vspace{-3pt}
\section{Conclusion} 
\label{sec:conclusion}

We introduce a parallel, differentiable reachability framework for closed-loop robotic systems with NN dynamics and controllers. By combining Taylor-model flowpipes with CROWN-style linear bound propagation in JAX, we obtain a GPU-accelerated reachability primitive for robot learning and planning. We further develop certified training and reachability-aware MPC, enabling verification-aware optimization of models and actions. Experiments across manipulation and quadrotor tasks demonstrate practical online planning with certified reachable-set over-approximations under uncertainty.

Our current framework primarily targets smooth or learned dynamics and still faces challenges from long-horizon conservatism and hybrid or contact-rich dynamics. We hope this work motivates future research on scalable differentiable reachability and tighter integration between formal verification, robot learning, and control.

\section*{Acknowledgments}

This work was supported in part by a Mathworks Microgrant. We would like to thank the CROWN-Reach team~\cite{CROWNReach2026}, especially Xiangru Zhong and Prof. Huan Zhang, for insightful discussions. We also thank Jeffrey Fang and Wei-Chen Li for their assistance with hardware setup.

%% Use plainnat to work nicely with natbib. 

\bibliographystyle{plainnat}
\bibliography{references}

@article{manzanas_lopez2024arch,
  title={Arch-comp24 category report: Artificial intelligence and neural network control systems (ainncs) for continuous and hybrid systems plants},
  author={Manzanas\_Lopez, Diego and Althoff, Matthias and Benet, Luis and Blab, Clemens and Forets, Marcelo and Jia, Yuhao and Johnson, Taylor T and Kranzl, Manuel and Ladner, Tobias and Linauer, Lukas and others},
  year={2024},
  publisher={Easychair}
}

@inproceedings{chen2013flow,
  title={Flow*: An analyzer for non-linear hybrid systems},
  author={Chen, Xin and {\'A}brah{\'a}m, Erika and Sankaranarayanan, Sriram},
  booktitle={International Conference on Computer Aided Verification},
  pages={258--263},
  year={2013},
  organization={Springer}
}

@inproceedings{chen2012taylor,
  title={Taylor model flowpipe construction for non-linear hybrid systems},
  author={Chen, Xin and Abraham, Erika and Sankaranarayanan, Sriram},
  booktitle={2012 IEEE 33rd Real-Time Systems Symposium},
  pages={183--192},
  year={2012},
  organization={IEEE}
}

@inproceedings{chen2016decomposed,
  title={Decomposed reachability analysis for nonlinear systems},
  author={Chen, Xin and Sankaranarayanan, Sriram},
  booktitle={2016 IEEE Real-Time Systems Symposium (RTSS)},
  pages={13--24},
  year={2016},
  organization={IEEE}
}

@phdthesis{chen2015reachability,
  title={Reachability analysis of non-linear hybrid systems using taylor models},
  author={Chen, Xin},
  year={2015},
  school={Fachgruppe Informatik, RWTH Aachen University}
}

@article{berz1998verified,
  title={Verified integration of ODEs and flows using differential algebraic methods on high-order Taylor models},
  author={Berz, Martin and Makino, Kyoko},
  journal={Reliable computing},
  volume={4},
  number={4},
  pages={361--369},
  year={1998},
  publisher={Springer}
}

@inproceedings{huang2022polar,
  title={POLAR: A polynomial arithmetic framework for verifying neural-network controlled systems},
  author={Huang, Chao and Fan, Jiameng and Chen, Xin and Li, Wenchao and Zhu, Qi},
  booktitle={International Symposium on Automated Technology for Verification and Analysis},
  pages={414--430},
  year={2022},
  organization={Springer}
}

@article{wang2023polar,
  title={Polar-express: Efficient and precise formal reachability analysis of neural-network controlled systems},
  author={Wang, Yixuan and Zhou, Weichao and Fan, Jiameng and Wang, Zhilu and Li, Jiajun and Chen, Xin and Huang, Chao and Li, Wenchao and Zhu, Qi},
  journal={IEEE Transactions on Computer-Aided Design of Integrated Circuits and Systems},
  volume={43},
  number={3},
  pages={994--1007},
  year={2023},
  publisher={IEEE}
}

@inproceedings{Althoff2015ARCH,
author			= {Matthias Althoff},
title			= {An Introduction to {CORA} 2015},
year			= {2015},
month 			= {December},
booktitle		= {Proc. of the 1st and 2nd Workshop on Applied Verification for Continuous and Hybrid Systems},
doi			= {10.29007/zbkv},
url			= {https://easychair.org/publications/paper/xMm},
publisher 		= {EasyChair},
pages			= {120-151}
}

@inproceedings{bogomolov2019juliareach,
  title={JuliaReach: a toolbox for set-based reachability},
  author={Bogomolov, Sergiy and Forets, Marcelo and Frehse, Goran and Potomkin, Kostiantyn and Schilling, Christian},
  booktitle={Proceedings of the 22nd ACM International Conference on Hybrid Systems: Computation and Control},
  pages={39--44},
  year={2019}
}

@inproceedings{tran2020nnv,
  title={NNV: the neural network verification tool for deep neural networks and learning-enabled cyber-physical systems},
  author={Tran, Hoang-Dung and Yang, Xiaodong and Manzanas Lopez, Diego and Musau, Patrick and Nguyen, Luan Viet and Xiang, Weiming and Bak, Stanley and Johnson, Taylor T},
  booktitle={International conference on computer aided verification},
  pages={3--17},
  year={2020},
  organization={Springer}
}

@inproceedings{lopez2023nnv,
  title={NNV 2.0: The neural network verification tool},
  author={Lopez, Diego Manzanas and Choi, Sung Woo and Tran, Hoang-Dung and Johnson, Taylor T},
  booktitle={International Conference on Computer Aided Verification},
  pages={397--412},
  year={2023},
  organization={Springer}
}

@misc{CROWNReach2026,
  title        = {CROWN-Reach: A Reachability Analysis Tool for Neural Network Controlled Systems},
  author       = {Zhong, Xiangru and Jia, Yuhao and Zhang, Huan},
  year         = {2024},
  howpublished = {\url{https://github.com/Verified-Intelligence/CROWN-Reach}},
  note         = {GitHub repository, accessed January 30, 2026},
}

@inproceedings{gruenbacher2022gotube,
  title={Gotube: Scalable statistical verification of continuous-depth models},
  author={Gruenbacher, Sophie A and Lechner, Mathias and Hasani, Ramin and Rus, Daniela and Henzinger, Thomas A and Smolka, Scott A and Grosu, Radu},
  booktitle={Proceedings of the AAAI Conference on Artificial Intelligence},
  volume={36},
  number={6},
  pages={6755--6764},
  year={2022}
}

@article{harapanahalli2024immrax,
  title={immrax: A parallelizable and differentiable toolbox for interval analysis and mixed monotone reachability in jax},
  author={Harapanahalli, Akash and Jafarpour, Saber and Coogan, Samuel},
  journal={IFAC-PapersOnLine},
  volume={58},
  number={11},
  pages={75--80},
  year={2024},
  publisher={Elsevier}
}

@phdthesis{kidger2021on,
    title={{O}n {N}eural {D}ifferential {E}quations},
    author={Patrick Kidger},
    year={2021},
    school={University of Oxford},
}

@software{jax2018github,
  author = {James Bradbury and Roy Frostig and Peter Hawkins and Matthew James Johnson and Chris Leary and Dougal Maclaurin and George Necula and Adam Paszke and Jake Vander{P}las and Skye Wanderman-{M}ilne and Qiao Zhang},
  title = {{JAX}: composable transformations of {P}ython+{N}um{P}y programs},
  url = {http://github.com/jax-ml/jax},
  version = {0.3.13},
  year = {2018},
}

@inproceedings{salman2019convex,
  title = {A Convex Relaxation Barrier to Tight Robustness Verification of Neural Networks},
  author = {Salman, Hadi and Yang, Greg and Zhang, Huan and Hsieh, Cho-Jui and Zhang, Pengchuan},
  booktitle = {Advances in Neural Information Processing Systems (NeurIPS)},
  year = {2019}
}

@inproceedings{zhang2018efficient,
  title={Efficient neural network robustness certification with general activation functions},
  author={Zhang, Huan and Weng, Tsui-Wei and Chen, Pin-Yu and Hsieh, Cho-Jui and Daniel, Luca},
  booktitle={Advances in Neural Information Processing Systems (NeurIPS)},
  year={2018}
}

@inproceedings{herbert2017fastrack,
  title={FaSTrack: A modular framework for fast and guaranteed safe motion planning},
  author={Herbert, Sylvia L and Chen, Mo and Han, SooJean and Bansal, Somil and Fisac, Jaime F and Tomlin, Claire J},
  booktitle={2017 IEEE 56th Annual Conference on Decision and Control (CDC)},
  pages={1517--1522},
  year={2017},
  organization={IEEE}
}

@article{zhang2023exact,
  title={Exact verification of relu neural control barrier functions},
  author={Zhang, Hongchao and Wu, Junlin and Vorobeychik, Yevgeniy and Clark, Andrew},
  journal={Advances in neural information processing systems},
  volume={36},
  pages={5685--5705},
  year={2023}
}

@article{majumdar2017funnel,
  title={Funnel libraries for real-time robust feedback motion planning},
  author={Majumdar, Anirudha and Tedrake, Russ},
  journal={The International Journal of Robotics Research},
  volume={36},
  number={8},
  pages={947--982},
  year={2017},
  publisher={SAGE Publications Sage UK: London, England}
}

@inproceedings{chou2021model,
  title={Model error propagation via learned contraction metrics for safe feedback motion planning of unknown systems},
  author={Chou, Glen and Ozay, Necmiye and Berenson, Dmitry},
  booktitle={2021 60th IEEE Conference on Decision and Control (CDC)},
  pages={3576--3583},
  year={2021},
  organization={IEEE}
}

@article{knuth2021planning,
  title={Planning with learned dynamics: Probabilistic guarantees on safety and reachability via lipschitz constants},
  author={Knuth, Craig and Chou, Glen and Ozay, Necmiye and Berenson, Dmitry},
  journal={IEEE Robotics and Automation Letters},
  volume={6},
  number={3},
  pages={5129--5136},
  year={2021},
  publisher={IEEE}
}

@article{singh2023robust,
  title={Robust feedback motion planning via contraction theory},
  author={Singh, Sumeet and Landry, Benoit and Majumdar, Anirudha and Slotine, Jean-Jacques and Pavone, Marco},
  journal={The International Journal of Robotics Research},
  volume={42},
  number={9},
  pages={655--688},
  year={2023},
  publisher={SAGE Publications Sage UK: London, England}
}

@article{knuth2022statistical,
  title={Statistical safety and robustness guarantees for feedback motion planning of unknown underactuated stochastic systems},
  author={Knuth, Craig and Chou, Glen and Reese, Jamie and Moore, Joe},
  journal={arXiv preprint arXiv:2212.06874},
  year={2022}
}

@inproceedings{chou2022safe,
  title={Safe output feedback motion planning from images via learned perception modules and contraction theory},
  author={Chou, Glen and Ozay, Necmiye and Berenson, Dmitry},
  booktitle={International Workshop on the Algorithmic Foundations of Robotics},
  pages={349--367},
  year={2022},
  organization={Springer}
}

@inproceedings{liu2023safe,
  title={Safe control under input limits with neural control barrier functions},
  author={Liu, Simin and Liu, Changliu and Dolan, John},
  booktitle={Conference on Robot Learning},
  pages={1970--1980},
  year={2023}
}

@inproceedings{yin2020optimization,
  title={Optimization based planner--tracker design for safety guarantees},
  author={Yin, He and Bujarbaruah, Monimoy and Arcak, Murat and Packard, Andrew},
  booktitle={2020 American Control Conference (ACC)},
  pages={5194--5200},
  year={2020},
  organization={IEEE}
}

@article{chi2025diffusion,
  title={Diffusion policy: Visuomotor policy learning via action diffusion},
  author={Chi, Cheng and Xu, Zhenjia and Feng, Siyuan and Cousineau, Eric and Du, Yilun and Burchfiel, Benjamin and Tedrake, Russ and Song, Shuran},
  journal={The International Journal of Robotics Research},
  volume={44},
  number={10-11},
  pages={1684--1704},
  year={2025},
  publisher={Sage Publications Sage UK: London, England}
}

@article{levine2016end,
  title={End-to-end training of deep visuomotor policies},
  author={Levine, Sergey and Finn, Chelsea and Darrell, Trevor and Abbeel, Pieter},
  journal={Journal of Machine Learning Research},
  volume={17},
  number={39},
  pages={1--40},
  year={2016}
}

@inproceedings{so2024train,
  title={How to train your neural control barrier function: Learning safety filters for complex input-constrained systems},
  author={So, Oswin and Serlin, Zachary and Mann, Makai and Gonzales, Jake and Rutledge, Kwesi and Roy, Nicholas and Fan, Chuchu},
  booktitle={2024 IEEE International Conference on Robotics and Automation (ICRA)},
  pages={11532--11539},
  year={2024},
  organization={IEEE}
}

@inproceedings{singh2018robust,
  title={Robust tracking with model mismatch for fast and safe planning: an sos optimization approach},
  author={Singh, Sumeet and Chen, Mo and Herbert, Sylvia L and Tomlin, Claire J and Pavone, Marco},
  booktitle={International Workshop on the Algorithmic Foundations of Robotics},
  pages={545--564},
  year={2018},
  organization={Springer}
}

@article{xu2020automatic,
  title={Automatic Perturbation Analysis for Scalable Certified Robustness and Beyond},
  author={Xu, Kaidi and Shi, Zhouxing and Zhang, Huan and Wang, Yihan and Chang, Kai-Wei and Huang, Minlie and Kailkhura, Bhavya and Lin, Xue and Hsieh, Cho-Jui},
  journal={Advances in Neural Information Processing Systems (NeurIPS)},
  year={2020}
}

@article{xu2020fast,
  title={Fast and complete: Enabling complete neural network verification with rapid and massively parallel incomplete verifiers},
  author={Xu, Kaidi and Zhang, Huan and Wang, Shiqi and Wang, Yihan and Jana, Suman and Lin, Xue and Hsieh, Cho-Jui},
  journal={International Conference on Learning Representations (ICLR)},
  year={2021}
}

@inproceedings{wang2021beta,
  title={Beta-CROWN: Efficient Bound Propagation with Per-neuron Split Constraints for Complete and Incomplete Neural Network Robustness Verification},
  author={Wang, Shiqi and Zhang, Huan and Xu, Kaidi and Jana, Suman and Lin, Xue and Hsieh, Cho-Jui and Kolter, Zico},
  booktitle={Advances in Neural Information Processing Systems (NeurIPS)},
  year={2021},
}

@article{zhang2022gcpcrown,
  title={General Cutting Planes for Bound-Propagation-Based Neural Network Verification},
  author={Zhang, Huan and Wang, Shiqi and Xu, Kaidi and Li, Linyi and Li, Bo and Jana, Suman and Hsieh, Cho-Jui and Kolter, J Zico},
  journal={Advances in Neural Information Processing Systems},
  year={2022}
}

@misc{tjeng2019evaluating,
      title={Evaluating Robustness of Neural Networks with Mixed Integer Programming}, 
      author={Vincent Tjeng and Kai Xiao and Russ Tedrake},
      year={2019},
      eprint={1711.07356},
      archivePrefix={arXiv},
      primaryClass={cs.LG}
}

@article{liu2021algorithms,
  title={Algorithms for verifying deep neural networks},
  author={Liu, Changliu and Arnon, Tomer and Lazarus, Christopher and Strong, Christopher and Barrett, Clark and Kochenderfer, Mykel J.},
  journal={Foundations and Trends\textregistered{} in Optimization},
  volume={4},
  number={3-4},
  year={2021}
}

@inproceedings{lu2020neural,
  title={Neural network branching for neural network verification},
  author={Lu, Jingyue and M. Pawan Kumar.},
  booktitle={International Conference on Learning Representations (ICLR)},
  year={2020}
}

@inproceedings{wong2018provable,
  title={Provable defenses against adversarial examples via the convex outer adversarial polytope},
  author={Eric Wong, J. Zico Kolter},
  booktitle={International Conference on Machine Learning (ICML)},
  year={2018}
}

@InProceedings{pmlr-v162-zhang22ae,
  title = 	 {A Branch and Bound Framework for Stronger Adversarial Attacks of {R}e{LU} Networks},
  author =       {Zhang, Huan and Wang, Shiqi and Xu, Kaidi and Wang, Yihan and Jana, Suman and Hsieh, Cho-Jui and Kolter, Zico},
  booktitle = 	 {International Conference on Machine Learning (ICML)},
  pages = 	 {26591--26604},
  year = 	 {2022},
  organization =    {PMLR},
}

@inproceedings{bunel2018unified,
  title={A unified view of piecewise linear neural network verification},
  author={Rudy Bunel and Ilker Turkaslan and Philip H. S. Torr and Pushmeet Kohli and M. Pawan Kumar},
  booktitle={Advances in Neural Information Processing Systems (NeurIPS)},
  year={2018}
}

@misc{blomqvist2022pymunk,
  author = {Blomqvist, Victor},
  title = {Pymunk},
  howpublished = {\url{https://pymunk.org}},
  month = nov,
  year = {2022},
  version = {6.4.0}
}

@article{gowal2018effectiveness,
  title={On the effectiveness of interval bound propagation for training verifiably robust models},
  author={Gowal, Sven and Dvijotham, Krishnamurthy and Stanforth, Robert and Bunel, Rudy and Qin, Chongli and Uesato, Jonathan and Mann, Timothy and Kohli, Pushmeet},
  journal={Proceedings of the IEEE International Conference on Computer Vision (ICCV)},
  year={2019}
}

@inproceedings{wang2018efficient,
  title={Efficient formal safety analysis of neural networks},
  author={Wang, Shiqi and Pei, Kexin and Whitehouse, Justin and Yang, Junfeng and Jana, Suman},
  booktitle={Advances in Neural Information Processing Systems (NeurIPS)},
  year={2018}
}

@article{singh2019abstract,
  title={An abstract domain for certifying neural networks},
  author={Singh, Gagandeep and Gehr, Timon and P{\"u}schel, Markus and Vechev, Martin},
  journal={Proceedings of the ACM on Programming Languages (POPL)},
  year={2019}
}

@article{brix2024fifth,
  title={The fifth international verification of neural networks competition (vnn-comp 2024): Summary and results},
  author={Brix, Christopher and Bak, Stanley and Johnson, Taylor T and Wu, Haoze},
  journal={arXiv preprint arXiv:2412.19985},
  year={2024}
}

@article{kaulen20256th,
  title={The 6th International Verification of Neural Networks Competition (VNN-COMP 2025): Summary and Results},
  author={Kaulen, Konstantin and Ladner, Tobias and Bak, Stanley and Brix, Christopher and Duong, Hai and Flinkow, Thomas and Johnson, Taylor T and Koller, Lukas and Manino, Edoardo and Nguyen, ThanhVu H and others},
  journal={arXiv preprint arXiv:2512.19007},
  year={2025}
}

@article{srinivasan2026safety,
  title={Safety beyond the training data: Robust out-of-distribution mpc via conformalized system level synthesis},
  author={Srinivasan, Anutam and Leeman, Antoine and Chou, Glen},
  journal={arXiv preprint arXiv:2602.12047},
  year={2026}
}

@article{shen2024bab,
  title={Bab-nd: Long-horizon motion planning with branch-and-bound and neural dynamics},
  author={Shen, Keyi and Yu, Jiangwei and Barreiros, Jose and Zhang, Huan and Li, Yunzhu},
  journal={arXiv preprint arXiv:2412.09584},
  year={2024}
}

@misc{shi2026certifiedtrainingbranchandboundlyapunovstable,
      title={Certified Training with Branch-and-Bound for Lyapunov-stable Neural Control}, 
      author={Zhouxing Shi and Haoyu Li and Cho-Jui Hsieh and Huan Zhang},
      year={2026},
      eprint={2411.18235},
      archivePrefix={arXiv},
      primaryClass={cs.LG},
      url={https://arxiv.org/abs/2411.18235}, 
}

@misc{wu2024verifiedsafereinforcementlearning,
      title={Verified Safe Reinforcement Learning for Neural Network Dynamic Models}, 
      author={Junlin Wu and Huan Zhang and Yevgeniy Vorobeychik},
      year={2024},
      eprint={2405.15994},
      archivePrefix={arXiv},
      primaryClass={cs.LG},
      url={https://arxiv.org/abs/2405.15994}, 
}

@article{yang2024lyapunov,
  title={Lyapunov-stable neural control for state and output feedback: A novel formulation},
  author={Yang, Lujie and Dai, Hongkai and Shi, Zhouxing and Hsieh, Cho-Jui and Tedrake, Russ and Zhang, Huan},
  journal={arXiv preprint arXiv:2404.07956},
  year={2024}
}

@article{li2025neural,
  title={Neural Contraction Metrics with Formal Guarantees for Discrete-Time Nonlinear Dynamical Systems},
  author={Li, Haoyu and Zhong, Xiangru and Hu, Bin and Zhang, Huan},
  journal={arXiv preprint arXiv:2504.17102},
  year={2025}
}

@article{li2025two,
  title={Two-Stage Learning of Stabilizing Neural Controllers via Zubov Sampling and Iterative Domain Expansion},
  author={Li, Haoyu and Zhong, Xiangru and Hu, Bin and Zhang, Huan},
  journal={arXiv preprint arXiv:2506.01356},
  year={2025}
}

@article{zhang2019towards,
  title={Towards stable and efficient training of verifiably robust neural networks},
  author={Zhang, Huan and Chen, Hongge and Xiao, Chaowei and Gowal, Sven and Stanforth, Robert and Li, Bo and Boning, Duane and Hsieh, Cho-Jui},
  journal={arXiv preprint arXiv:1906.06316},
  year={2019}
}

@misc{shi2023robocook,
      title={RoboCook: Long-Horizon Elasto-Plastic Object Manipulation with Diverse Tools}, 
      author={Haochen Shi and Huazhe Xu and Samuel Clarke and Yunzhu Li and Jiajun Wu},
      year={2023},
      eprint={2306.14447},
      archivePrefix={arXiv},
      primaryClass={cs.RO}
}

@article{huang2021training,
  title={Training certifiably robust neural networks with efficient local lipschitz bounds},
  author={Huang, Yujia and Zhang, Huan and Shi, Yuanyuan and Kolter, J Zico and Anandkumar, Anima},
  journal={Advances in Neural Information Processing Systems},
  volume={34},
  pages={22745--22757},
  year={2021}
}

@inproceedings{Wang-RSS-23,
	address = {Daegu, Republic of Korea},
	author = {Yixuan Wang AND Yunzhu Li AND Katherine Driggs-Campbell AND Li Fei-Fei AND Jiajun Wu},
	booktitle = {Proceedings of Robotics: Science and Systems},
	doi = {10.15607/RSS.2023.XIX.047},
	month = {July},
	title = {{Dynamic-Resolution Model Learning for Object Pile Manipulation}},
	year = {2023}}

@article{finn2016unsupervised,
  title={Unsupervised learning for physical interaction through video prediction},
  author={Finn, Chelsea and Goodfellow, Ian and Levine, Sergey},
  journal={arXiv preprint arXiv:1605.07157},
  year={2016}
}

@article{ebert2018visual,
  title={Visual foresight: Model-based deep reinforcement learning for vision-based robotic control},
  author={Ebert, Frederik and Finn, Chelsea and Dasari, Sudeep and Xie, Annie and Lee, Alex and Levine, Sergey},
  journal={arXiv preprint arXiv:1812.00568},
  year={2018}
}

@article{hafner2019dream,
  title={Dream to control: Learning behaviors by latent imagination},
  author={Hafner, Danijar and Lillicrap, Timothy and Ba, Jimmy and Norouzi, Mohammad},
  journal={arXiv preprint arXiv:1912.01603},
  year={2019}
}

@inproceedings{wu2023daydreamer,
  title={Daydreamer: World models for physical robot learning},
  author={Wu, Philipp and Escontrela, Alejandro and Hafner, Danijar and Abbeel, Pieter and Goldberg, Ken},
  booktitle={Conference on Robot Learning},
  pages={2226--2240},
  year={2023},
  organization={PMLR}
}

@article{shi2022robocraft,
  title={RoboCraft: Learning to See, Simulate, and Shape Elasto-Plastic Objects with Graph Networks},
  author={Shi, Haochen and Xu, Huazhe and Huang, Zhiao and Li, Yunzhu and Wu, Jiajun},
  journal={arXiv preprint arXiv:2205.02909},
  year={2022}
}

@article{manuelli2020keypoints,
  title={Keypoints into the Future: Self-Supervised Correspondence in Model-Based Reinforcement Learning},
  author={Manuelli, Lucas and Li, Yunzhu and Florence, Pete and Tedrake, Russ},
  journal={arXiv preprint arXiv:2009.05085},
  year={2020}
}

@article{li2018learning,
  title={Learning particle dynamics for manipulating rigid bodies, deformable objects, and fluids},
  author={Li, Yunzhu and Wu, Jiajun and Tedrake, Russ and Tenenbaum, Joshua B and Torralba, Antonio},
  journal={arXiv preprint arXiv:1810.01566},
  year={2018}
}

@ARTICLE{6289431,
  author={Mahony, Robert and Kumar, Vijay and Corke, Peter},
  journal={IEEE Robotics \& Automation Magazine}, 
  title={Multirotor Aerial Vehicles: Modeling, Estimation, and Control of Quadrotor}, 
  year={2012},
  volume={19},
  number={3},
  pages={20-32},
  doi={10.1109/MRA.2012.2206474}}

@misc{jax_verify2020,
  title        = {jax\_verify: Neural Network Verification in JAX},
  author       = {{Google DeepMind}},
  year         = {2020},
  howpublished = {\url{https://github.com/google-deepmind/jax_verify}},
  note         = {GitHub repository},
}

@book{rubinstein2013cross,
  title={The cross-entropy method: a unified approach to combinatorial optimization, Monte-Carlo simulation and machine learning},
  author={Rubinstein, Reuven Y and Kroese, Dirk P},
  year={2013},
  publisher={Springer Science \& Business Media}
}

@article{williams2017model,
  title={Model predictive path integral control: From theory to parallel computation},
  author={Williams, Grady and Aldrich, Andrew and Theodorou, Evangelos A},
  journal={Journal of Guidance, Control, and Dynamics},
  volume={40},
  number={2},
  pages={344--357},
  year={2017},
  publisher={American Institute of Aeronautics and Astronautics}
}

@INPROCEEDINGS{11312100,
  author={Serry, Mohamed and Li, Haoyu and Zhou, Ruikun and Zhang, Huan and Liu, Jun},
  booktitle={2025 IEEE 64th Conference on Decision and Control (CDC)}, 
  title={Safe Domains of Attraction for Discrete-Time Nonlinear Systems: Characterization and Verifiable Neural Network Estimation}, 
  year={2025},
  volume={},
  number={},
  pages={5774-5781},
  doi={10.1109/CDC57313.2025.11312100}}

@inproceedings{ARCH-COMP24,
  author    = {Diego Manzanas Lopez and Matthias Althoff and Luis Benet and Clemens Blab and Marcelo Forets and Yuhao Jia and Taylor T Johnson and Manuel Kranzl and Tobias Ladner and Lukas Linauer and Philipp Neubauer and Sophie Neubauer and Christian Schilling and Huan Zhang and Xiangru Zhong},
  title     = {ARCH-COMP24 Category Report: Artificial Intelligence and Neural Network Control Systems (AINNCS) for Continuous and Hybrid Systems Plants},
  booktitle = {Proceedings of the 11th Int. Workshop on Applied Verification for Continuous and Hybrid Systems},
  editor    = {Goran Frehse and Matthias Althoff},
  series    = {EPiC Series in Computing},
  volume    = {103},
  publisher = {EasyChair},
  bibsource = {EasyChair, https://easychair.org},
  issn      = {2398-7340},
  url       = {/publications/paper/WsgX},
  doi       = {10.29007/mxld},
  pages     = {64-121},
  year      = {2024}}

\newpage
\clearpage
\newpage
\begin{appendices}
\newpage

% Change section and subsection title 
\renewcommand{\thesection}{\Alph{section}} % Changes section numbering to A, B, C, etc.

\renewcommand{\thesubsection}{\Alph{section}.\arabic{subsection}}

\section{Reachability of CT analytical dynamics}\label{app:reach_ct_ana_dyn}

% \newpage
\begin{algorithm}[t]
    \caption{Reachability of CT Dynamics with TM Flowpipe. \textcolor{brown}{Comments} are in brown.}
    \label{alg:ct-tm-flowpipe}
    {\small
    \begin{algorithmic}[1]
        \State \textbf{Inputs:} dynamics \(\ctdyn\) (analytical or NN), initial box \(\gX_0=[\underline{x}_0,\overline{x}_0]\), step size \(h\), number of steps \(N\), TM order \(k\), initial remainder estimate \(\epsilon_{\mathrm{init}}\), remainder refinement rounds \(R\), auxiliary-window size \(M\)
        \State \textbf{Outputs:} certified reachable set \(\{\gX_i\}_{i=0}^{N}=\{\underline{x}_i,\overline{x}_i\}_{i=0}^{N}\)
        \vspace{0.5em}

        \vspace{0.5em}
        \State \textcolor{brown}{\# Build initial TM for the initial set: \(x = c + S y\)}
        \State \([\underline{x}_0,\overline{x}_0] \gets \gX_0\)
        \State \(c_0 \gets (\underline{x}_0+\overline{x}_0)/2,\quad S_0 \gets (\overline{x}_0-\underline{x}_0)/2\)
        \State \(X^{(0)} \gets \mtt{BuildLinearTM}(c_0,S_0)\) \Comment{\textcolor{brown}{affine TM over unit box}}
        \State \(\gT^{(0)} \gets \mtt{IdentityTM}()\) \Comment{\textcolor{brown}{initial linear parameterization}}
        \State \(\mathcal{S}^{(0)} \gets \mtt{InitSymbolicState}(M)\) \Comment{\textcolor{brown}{wrapping-control state}}

        \vspace{0.5em}
        \State \textcolor{brown}{\# Main loop}
        \For{$i=0,\dots,N-1$}
            \State \textcolor{brown}{\# (A) Preconditioning / linear reparameterization}
            \State \(\tilde{X} \gets X^{(i)}\big|_{t:=h}\) \Comment{\textcolor{brown}{evaluate previous step at end time}}
            \State \((c, S,\;\gT^{(i+1)},\;\mathcal{S}^{(i+1)}) \gets \mtt{SymbolicStep}(\gT^{(i)},\tilde{X},\mathcal{S}^{(i)})\)
            \State \(X \gets \mtt{BuildLinearTM}(c,S)\) \Comment{\textcolor{brown}{new affine seed \(x=c+Sy\)}}

            \vspace{0.5em}
            \State \textcolor{brown}{\# (B) Polynomial-only Picard iteration}
            \State \(p_k \gets \mtt{PolyPicard}(\ctdyn,\;X,\;h,\;k)\)

            \State \textcolor{brown}{\# (C) Seed and refine the remainder in Picard (soundness)}
            \State \(X^{(i+1)} \gets \mtt{RemainderPicard}\big(\ctdyn,\;X,\;p_k,\;h,\;\epsilon_{\mathrm{init}},\;R\big)\)

            \vspace{0.5em}
            \State \textcolor{brown}{\# (D) Certified reachable set over the step segment}
            \State \(X_{\mathrm{seg}} \gets X^{(i+1)} \circ \gT^{(i+1)}\)
            \State \([\underline{x}_{i+1},\overline{x}_{i+1}] \gets \mtt{EvalNormInterval}(X_{\mathrm{seg}},\;h)\)
        \EndFor
        \vspace{0.5em}
        \State \Return \(\{\underline{x}_i,\overline{x}_i\}_{i=0}^{N}\)
    \end{algorithmic}
    }
\end{algorithm}

Algorithm~\ref{alg:ct-tm-flowpipe} constructs a certified flowpipe for \(\dot{x}(\tau)=\ctdyn(x(\tau))\) by propagating a TM enclosure over \(N\) fixed-duration segments. It first maps the initial box \(\gX_0=[\underline{x}_0,\overline{x}_0]\) to a normalized affine TM \(X^{(0)}=c_0+S_0y\) where $y\in[-1,1]^n$ and $c_0$ and $S_0$ are corresponding constant and linear coefficient (Lines 4--6). This normalization avoids unnecessary conservativeness when evaluating intervals over large and non-symmetric (zero-centered) input domains. It also defines a linear parameterization map $\gT^{(i)}$ that tracks the cumulative normalization (preconditioning) applied on every step, which is initially an identity TM $\gT^{(0)}$ (Line 7). The symbolic state $\gS^{(i)}$ will be used to reduce the overestimation when propagating flowpipes step by step.

Each step then advances the enclosure with four operations. 
\begin{itemize}
    \item (A) \emph{Preconditioning / reparameterization} (Lines 12-14) evaluates the previous step at \(\tau=h\) to obtain an endpoint TM \(\tilde{X}\), and computes a new affine seed \(X_0=c+Sy\) maintaining a short symbolic remainder state to keep computation tractable and subsequent enclosures tight. 
    \item (B) \emph{Polynomial Picard} (Line 16) applies truncated Picard iteration in TM arithmetic to obtain a \(k\)-order polynomial approximation \(p_k\) of the local flow map over \([-1,1]^n\). 
    \item (C) \emph{Remainder Picard} (Line 18) then seeds and refines the interval remainder around \(p_k\) until the enclosure is self-consistent under the Picard update, yielding a sound TM flowpipe segment \(X^{(i+1)}\). 
    \item Finally, (D) the algorithm composes the segment with the accumulated linear parameterization \(\gT^{(i+1)}\) and evaluates it over \([-1,1]^n\) to extract certified box bounds \([\underline{x}_{i+1},\overline{x}_{i+1}]\) as reachable sets for use downstream in planning and learning (Lines 20,21).
\end{itemize}

\section{Reachability of CT Neural dynamics}\label{app:reach_ct_nn_dyn}

In this section, we aim to prove~\Theoref{theo:crown_tm}. \begin{proof}
Let the linear TM input be
\[
g(z)=c_g + A_g z \oplus I_g,
\quad z\in\gZ,\; I_g=[\lowerI_g,\upperI_g].
\]
Equivalently, for every admissible input \(x\in g(\gZ)\), there exist \(z\in\gZ\) and \(r\in I_g\) such that
\[
x = c_g + A_g z + r .
\]
Define the reparameterized network
\[
\tilde f(z,r) := f_\text{NN}(c_g + A_g z + r),
\]
with inputs \((z,r)\in \gZ\times I_g\). By construction,
\[
f_\text{NN}(g(z)) = \tilde f(z,r)
\quad \text{for some } r\in I_g.
\]

Applying CROWN to \(\tilde f\) over the box domain \(\gZ\times I_g\), with the enforced shared-slope constraint \(\lowerW=\upperW\), yields affine bounds
\[
\tileW_z z + \tileW_r r + \tilde b
\;\le\;
\tilde f(z,r)
\;\le\;
\tileW_z z + \tileW_r r + \tilde b,
\]
where \(\tileW_z\in\sR^{n_o\times n_z}\), \(\tileW_r\in\sR^{n_o\times n_i}\), and \(\tilde b\in\sR^{n_o}\).
Since the slopes are shared, the only remaining uncertainty is due to \(r\in I_g\). Over-approximating the contribution of \(r\) by interval arithmetic gives
\[
\tileW_r r \in 
\tileW_r I_g
= [\lowerI,\upperI],
\]
where \(\lowerI,\upperI\) are obtained by elementwise interval evaluation.
Now we define
\[
P := \tileW_z,
\qquad
c := \tilde b,
\qquad
I := [\lowerI,\upperI].
\]
Then for all \(z\in\gZ\),
\[
f_\text{NN}(g(z)) \in c + Pz \oplus I,
\]
which is a linear Taylor model that soundly over-approximates \(f_\text{NN}(g(z))\).
\end{proof}

\section{Certified training algorithm}\label{app:training}

We present our certified training algorithm for DT neural dynamics discussed in~\Secref{subsec:training} in~\Algref{alg:train_dt_dyn}. It includes multi-step loss, reachability regularization and curriculum schedules and can be adapted to train CT neural dynamics and controllers. 
% Please refer to~\Secref{subsec:training} for more detail.

\begin{algorithm}[t]
    \caption{Training of DT Neural Dynamics. 
    % \textcolor{brown}{Comments} are in brown.
    }
    \label{alg:train_dt_dyn}
    {\small
    \begin{algorithmic}[1]
        \State \textbf{Inputs:} dataset \(\gD=\{(x_t^m,u_t^m)\}\), model \(\dtdyn^{\theta}\), target horizon \(T_h^{\max}\), target radius \(\epsilon^{\mathrm{final}}\), weights \(\{w_t\}\), penalty weight \(\lambda\), optimizer \(\mtt{Opt}\), DT reachability engine \(\mtt{DTReach}(\cdot)\)
        \State \textbf{Outputs:} trained parameters \(\theta\)
        \vspace{0.5em}

        \State Reformat \(\gD\) into \(M'\) episodes \(\{(x_{0:T_h^{\max}}^m,u_{0:T_h^{\max}-1}^m)\}_{m=0}^{M'-1}\)
        \For{training iteration \(s=0,\dots,S-1\)}
            \State \textcolor{brown}{\# Curriculum schedules (horizon \& uncertainty)}
            \State \(T_h \gets \mtt{HorizonSchedule}(s;\,T_h^{\max})\) \Comment{\textcolor{brown}{logarithmic increase}}
            \State \(\epsilon \gets \mtt{EpsSchedule}(s;\,\epsilon^{\mathrm{final}},T_h)\) \Comment{\textcolor{brown}{smooth decrease}}

            \State Sample a mini-batch \(\mathcal{B}\subset\{0,\dots,M'-1\}\)
            \State \(\gL_{\text{pred}} \gets 0,\quad \gL_{\text{reach}} \gets 0\)
            \For{each episode \(m\in\mathcal{B}\)}
                \State \textcolor{brown}{\# (A) Autoregressive rollout for prediction loss}
                \State \(\hx_0^m \gets x_0^m\)
                \For{$t=0,\dots,T_h-1$}
                    \State \(\hx_{t+1}^m \gets \dtdyn^{\theta}(\hx_t^m,u_t^m)\)
                    \State \(\gL_{\text{pred}} \mathrel{+}= w_t\|\hx_{t+1}^m-x_{t+1}^m\|_2^2\)
                \EndFor

                \State \textcolor{brown}{\# (B) DT reachability for reachability loss}
                \State \(\gX_0^m \gets \gB_{\epsilon}(x_0^m)\)
                \State \(\{\gX_t^m\}_{t=1}^{T_h} \gets \mtt{DTReach}\big(\dtdyn^{\theta},\gX_0^m,u_{0:T_h-1}^m\big)\)
                \State \(V_{\gR,\epsilon}^m \gets \sum_{t=1}^{T_h}\sum_{j=0}^{n-1}\mtt{width}(\gX_{t,j}^m)\)
                \State \(\gL_{\text{reach}} \mathrel{+}= \log(1+V_{\gR,\epsilon}^m)\)
            \EndFor

            \State \textcolor{brown}{\# (C) Combine losses and update parameters}
            \State Normalize: \(\gL_{\text{pred}} \gets \gL_{\text{pred}}/(|\mathcal{B}|T_h)\), \(\gL_{\text{reach}} \gets \gL_{\text{reach}}/|\mathcal{B}|\)
            \State \(\gL_{\text{total}} \gets \gL_{\text{pred}} + \lambda\,\gL_{\text{reach}}\)
            \State \(\theta \gets \mtt{OptUpdate}(\mtt{Opt},\theta,\nabla_{\theta}\gL_{\text{total}})\)
        \EndFor
        \State \Return \(\theta\)
    \end{algorithmic}
    }
\end{algorithm}

\section{Additional experimental details}\label{app:exp}

\paragraph{T-pushing}
We define the state of the T-pushing task as the concatenation of 3D T object pose (2D position and 1D orientation of the center of mass of the T object) and 2D position of the cylinder pusher. We train an MLP with hidden layers $[96,96,96]$ as the DT neural dynamics. We then train an MLP with hidden layers $[64,64]$ as the CT neural dynamics. We finally train the neural controller (an MLP with hidden layers $[64,64]$) with fixed CT neural dynamics for multi-step rollout. The action of this task is defined as the 2D velocity of the cylinder pusher. We collect our training data from the simulator Pymunk~\cite{blomqvist2022pymunk}.

\paragraph{Quadrotor}

We use a standard rigid-body quadrotor model~\cite{6289431} with world-frame translation, ZYX Euler angles, and body-frame angular rates.

The state is
\[
x =
\begin{bmatrix}
x \quad y \quad  z \quad  \dot x \quad  \dot y \quad  \dot z \quad 
\phi \quad  \theta \quad  \psi \quad p \quad  q \quad  r
\end{bmatrix}^\top \in \mathbb{R}^{12},
\]
where $(x,y,z)$ are world positions, $(\dot x,\dot y,\dot z)$ are world velocities,
$(\phi,\theta,\psi)$ are roll, pitch, and yaw (ZYX convention), and $(p,q,r)$ are body angular rates.
The control input is
\[
u =
\begin{bmatrix}
u_1 & u_2 & u_3 & u_4
\end{bmatrix}^\top,
\]
where $u_1$ is the total thrust magnitude along the body $z$-axis and $(u_2,u_3,u_4)$ are body torques. In our experiments, the yaw torque input is fixed to zero ($u_4 = 0$),  removing active yaw control, which is sufficient for the planar and position-focused maneuvers considered in this work.

Its position kinematics are as follows.
\[
\begin{aligned}
\dot x &= \dot x, \\
\dot y &= \dot y, \\
\dot z &= \dot z .
\end{aligned}
\]

Let $R_{WB}$ denote the body-to-world rotation matrix under ZYX Euler angles and
\[
b_3 = R_{WB} e_3 =
\begin{bmatrix}
\cos\phi\sin\theta\cos\psi + \sin\phi\sin\psi \\
\cos\phi\sin\theta\sin\psi - \sin\phi\cos\psi \\
\cos\phi\cos\theta
\end{bmatrix}.
\]
The translational dynamics in the world frame are
\[
\begin{aligned}
\ddot x &= \frac{u_1}{m} b_{3x}, \\
\ddot y &= \frac{u_1}{m} b_{3y}, \\
\ddot z &= \frac{u_1}{m} b_{3z} - g ,
\end{aligned}
\]
where $m$ is the mass and $g$ is gravitational acceleration.

The Euler angle rates driven by body angular rates $(p,q,r)$ are
\[
\begin{aligned}
\dot\phi &= p + \sin\phi \tan\theta \, q + \cos\phi \tan\theta \, r, \\
\dot\theta &= \cos\phi \, q - \sin\phi \, r, \\
\dot\psi &= \frac{\sin\phi}{\cos\theta} q + \frac{\cos\phi}{\cos\theta} r .
\end{aligned}
\]
% This representation has a singularity at $\theta = \pm \frac{\pi}{2}$.

Assuming a diagonal inertia matrix $J=\mathrm{diag}(J_x,J_y,J_z)$, the rigid-body rotational dynamics are
\[
\begin{aligned}
\dot p &= \frac{J_y - J_z}{J_x} q r + \frac{u_2}{J_x}, \\
\dot q &= \frac{J_z - J_x}{J_y} p r + \frac{u_3}{J_y}, \\
\dot r &= \frac{J_x - J_y}{J_z} p q + \frac{u_4}{J_z}.
\end{aligned}
\]

We train an MLP with hidden layers $[64,64]$ as the DT dynamics for planning. Its state is defined as the 3D position of the quadrotor without including all remaining states in GT analytical ODE dynamics. We apply velocity control with a neural controller (an MLP with hidden layers $[64,64]$) accepting the desired 3D velocity command as the input and output CT actions $u$. We collect training data from the analytical dynamics with a nominal PD controller.

\end{appendices}

\end{document}